%% file: Arxiv_paper.tex
\documentclass{article}
\usepackage[utf8]{inputenc} 
\usepackage[T1]{fontenc}    
\usepackage{hyperref}       
\usepackage{url}            
\usepackage{booktabs}       
\usepackage{amsfonts}       
\usepackage{nicefrac}       
\usepackage{microtype}      
\usepackage{geometry}
\usepackage{subfig}
\usepackage{bbm}  

\usepackage[round]{natbib}

\input{prel.tex}

\usepackage{todonotes}

\newcommand{\comment}[1]{\ensuremath{\triangle~}\emph{#1}}

\newcommand{\OPT}{\mathrm{OPT}}

\newcommand{\phia}{\phi_{\alpha}}
\newcommand{\ka}{{k_+}}

\newcommand{\smallc}{I_{\mathrm{small}}^\alpha}
\newcommand{\smallcmax}{I_{\mathrm{small}}^{\alpham}}

\newcommand{\largec}
{I_{\mathrm{large}}^\alpha}
\newcommand{\smallp}
{C_{\mathrm{small}}^\alpha}
\newcommand{\smallpmax}
{C_{\mathrm{small}}^{\alpham}}
\newcommand{\largep}
{C_{\mathrm{large}}^\alpha}
\newcommand{\largepmax}
{C_{\mathrm{large}}^{\alpham}}

\newcommand{\ca}{T_{\alpha}}
\newcommand{\eo}{\psi}

\newcommand{\Tout}{T_{\mathrm{out}}}
\newcommand{\tr}{\tau}
\newcommand{\iflarge}{M}

\newcommand{\bins}{\cB}
\newcommand{\binsa}{\cB_a}
\newcommand{\binsb}{\cB_b}
\newcommand{\binsc}{\cB_c}
\newcommand{\far}{\mathrm{far}}
\newcommand{\fu}{\mathrm{F}_b}
\newcommand{\fl}{\mathrm{F}_c}
\newcommand{\fs}{\mathrm{F}_d}
\newcommand{\event}{E}

\newcommand{\algname}{\texttt{SelectProc}}
\newcommand{\finalalgname}{\texttt{MUNSC}}
\newcommand{\undert}{\texttt{Near}}
\newcommand{\alpham}{{\alpha_{I+1}}}
\newcommand{\phiam}{\phi_{\alpham}}

\makeatletter
\newcommand{\manuallabel}[2]{\def\@currentlabel{#2}\label{#1}}
\makeatother

\newcommand{\thealgorithms}{
\begin{algorithm}[t]
\caption{The sequential clustering algorithm: \finalalgname} \label{alg:final_alg}
\begin{algorithmic}[1]
\REQUIRE $\delta \in (0,1)$, $k\in \nats$, $n\in \nats$ (stream length), $\mathcal{A}$ (an offline $k$-median algorithm),\\ sequential access to the input stream.
\STATE $\alpha_1 \leftarrow \delta/(4k)$; For $i \in \nats$, $\alpha_i \leftarrow \alpha_1 \cdot 2^{i-1}$.
 \STATE $I \leftarrow \floor{\log_2(1/(6\alpha_1))}, \delta' \leftarrow \delta/(I+1)$.~~~\comment{$I$ is set so that $\alpha_{I+1} \in (1/12,1/6]$}.

\STATE Prepare $I+1$ copies of \algname, indexed by $i \in [I+1]$, as follows:\\\hspace{1em}In all copies, use inputs $k,n,\mathcal{A}$, $\tr \leftarrow \phiam$, and set the confidence parameter to $\delta'$.\\
\hspace{1em}In copy $i\in [I+1]$, set $\alpha \leftarrow \alpha_i$.\\
\hspace{1em}In all but the last copy, set $\iflarge \leftarrow 1$ and $\gamma \leftarrow 2\alpha_i$.\\

\hspace{1em}In the last copy (index $I+1$), set $\iflarge \leftarrow \log (8\ka/\delta')$, and $\gamma \leftarrow 1-2\alpha_{I+1}$.
\STATE Read each point from the input stream and feed it to each of the copies of \algname. \\If any of the copies selected the point, then select it as a center.
\end{algorithmic}
\end{algorithm}

\begin{algorithm}[t]
\caption{Internal procedure: \algname} \label{alg:basic_algorithm}
\begin{algorithmic}[1]
\REQUIRE $\delta \in (0,1)$, $k\in \nats$, $n\in \nats$, $\mathcal{A}$ (an offline $k$-median algorithm),\\
sequential access to the input stream.\\
Technical parameters: $\alpha \in (0,\frac{1}{6}]$, $\gamma \in (\alpha,1-2\alpha]$, $\iflarge \in \nats$, $\tr > 0$.
\STATE \comment{\textbf{Phase 1} (Calculate an initial clustering $\ca$)}:
\STATE $P_1 \leftarrow$ the first $\alpha n$  points from the input stream.
\STATE  $\ca := \{c_1,\ldots c_{\ka} \} \leftarrow \cA(P_1,\ka)$ ~~~\comment{Only calculation, no centers are actually selected here.}
\STATE \comment{\textbf{Phase 2} (Estimate the risk of $\ca$ on large optimal clusters)}:
\STATE $P_2\leftarrow$   next $\alpha n$  points  from the input stream. 
\STATE $\eo \leftarrow \frac{1}{3 \alpha } R_{2\alpha(k+1) \phia  }(P_2,\ca)$ ~~~\comment{$\eo$ estimates the risk of $\ca$ on large optimal clusters}. \label{line-P2}

\STATE \comment{\textbf{Phase 3} (Select centers)}:
\STATE $\forall i \in [\ka], \,\, n_i \leftarrow 0$, $\undert_i \leftarrow \texttt{FALSE}.$ ~~~\comment{$n_i$ counts selected points associated with $c_i$. \\\hspace{17em}$\undert_i$ indicates if a point close to $c_i$ was selected.}
\FOR{ $\gamma n$ iterations }
\STATE Read the next point $x$ from the input stream.

\STATE $i \leftarrow \argmin_{i \in [\ka]}\rho(x,c_i)$. ~~~\comment{Find the closet center to $x$ in $\ca$.}

\IF{$\rho(x,c_i) > \frac{\eo}{k \tr}$  or $\neg\undert_i$ or $n_i \leq \iflarge$}\label{line-nv-thr}  
\STATE \textbf{Select} $x$ as a center. \label{line-nv-select-x} ~~~~\comment{This is the only line in which actual selections occur.}
  \STATE $n_i \leftarrow n_i+1$. \label{line-nv-increase-n_i}
\STATE \textbf{If}  $\rho(x,c_i) \leq \frac{\eo}{k \tr}$ \textbf{then} $\undert_i \leftarrow \texttt{TRUE}$. \label{line-nv-increasingN_and_bi}
\ENDIF
\ENDFOR
\end{algorithmic}
\end{algorithm}
}

\title{A Constant Approximation Algorithm for Sequential Random-Order No-Substitution $k$-Median Clustering}

\usepackage{authblk}

\author[1]{Tom Hess}
\author[2]{Michal Moshkovitz}
\author[3]{Sivan Sabato}
\affil[1,3]{Department of Computer Science, Ben Gurion University, Beer-Sheva, Israel}
\affil[2]{Qualcomm Institute, University of California, San Diego, USA}
{
    \makeatletter
    \renewcommand\AB@affilsepx{: \protect\Affilfont}
    \makeatother

    \makeatletter
    \renewcommand\AB@affilsepx{, \protect\Affilfont}
    \makeatother
    \affil[1]{\normalsize \texttt{tomhe@bgu.ac.il}}
    \affil[2]{\normalsize \texttt{mmoshkovitz@eng.ucsd.edu}}
    \affil[3]{\normalsize \texttt{sabatos@cs.bgu.ac.il}}
}
\date{}

\begin{document}

\maketitle

\begin{abstract}
We study $k$-median clustering under the sequential no-substitution setting. In this setting, a data stream is sequentially observed, and some of the points are selected by the algorithm as cluster centers. However, a point can be selected as a center only immediately after it is observed, before observing the next point. In addition, a selected center cannot be substituted later. We give the first algorithm for this setting that  obtains a constant approximation factor on the optimal risk under a random arrival order, an exponential improvement over previous work. This is also the first constant approximation guarantee that holds without any structural assumptions on the input data. Moreover, the number of selected centers is only quasi-linear in~$k$.  Our algorithm and analysis are based on a careful risk estimation that avoids outliers, a new concept of a linear bin division, and a multiscale approach to center selection. 
\end{abstract}

\section{Introduction}\label{sec:intro}
 Clustering is a fundamental unsupervised learning task used for various applications, such as  anomaly detection \citep{leung2005unsupervised}, recommender systems \citep{shepitsen2008personalized} and cancer diagnosis \citep{zheng2014breast}. 

In recent years, the problem of \emph{sequential clustering} has been actively studied, motivated by applications in which data arrives sequentially, such as 
online recommender systems \citep{nasraoui2007performance} and online community detection \citep{aggarwal2003framework}.
 
In this work, we study $k$-median clustering in the sequential \emph{no-substitution} setting, a term first introduced in \cite{hess2020sequential}. In this setting, a stream of data points is sequentially observed, and some of these points are selected by the algorithm as cluster centers. However, a point can be selected as a center only immediately after it is observed, before observing the next point. In addition, a selected center cannot be substituted later. This setting is motivated by applications in which center selection is mapped to a real-world irreversible action. For instance, consider a stream of website users, where the goal is to instantaneously identify users that serve as social cluster centers and provide them with a promotional gift while they are still on the website. As another example, consider recruiting representative participants for a clinical trial out of a stream of incoming patients.

The goal in the no-substitution $k$-median setting is to obtain a near-optimal $k$-median risk value, while selecting a number of centers that is as close as possible to $k$.
For an adversarially ordered stream, it has been shown \citep{moshkovitz2021unexpected} that an algorithm that selects a number of centers that is sublinear in the stream length cannot obtain a constant approximation guarantee, unless some structural assumptions are imposed on the input data.

In this work, we show that in contrast, when the stream order is random, a constant approximation guarantee can be obtained without imposing structural assumptions on the data. Previous works on the random-order setting provide either an algorithm with an almost-constant approximation assuming a bounded metric space \citep{hess2020sequential} or an algorithm with an approximation factor that is exponential in $k$ \citep{moshkovitz2021unexpected}. Thus, our guarantees are  stronger and more general than the previous guarantees. Moreover, the number of centers selected by our algorithm is only quasi-linear in $k$ and does not depend on the stream length.

\paragraph{Main result.} 
We propose a new algorithm, \finalalgname\ (Multiscale No-Substitution Clustering),  
that obtains a constant approximation factor with probability $1-\delta$ over the random order of the stream, while selecting only $O(k \log^2(k/\delta))$ centers.
\finalalgname\ operates by executing several instances of a new center-selection procedure, called \algname, where each instance is applied to a stream prefix of a different scale. 
\algname\ decides which centers to select using a novel technique, in which a small stream prefix is used to estimate a truncated version of the optimal risk of the entire stream. This truncated version ignores outliers, and so can be estimated reliably. The value of the estimate is used to determine which centers to select from the rest of the stream. The multiscale use of \algname\ by \finalalgname\ allows \finalalgname\ to select only a quasi-linear number of centers. 

As part of our analysis, we propose a new concept of a \emph{linear bin division}, which allows
a high-probability association of the risks of stream subsets, while selecting a number of centers that is independent of the stream length. This improves a previous construction of \cite{meyerson2004k}.

The guarantees for \finalalgname\ are provided in our main theorem, \thmref{final-algorithm}, stated in \secref{algorithms}. \tabref{comparison} below compares our guarantees to previous works; See \secref{related_works} for additional details. We note that our algorithm and guarantees are easily extendable to $k$-means clustering, as well as to any other clustering objective that satisfies the weak triangle inequality. \finalalgname\ uses as a black box an offline $k$-median approximation algorithm with a constant approximation factor. Whenever the offline algorithm is efficient \citep[e.g.,][]{charikar2002constant,li2016approximating}, then so is \finalalgname.

\paragraph{Limitations.}
The proposed algorithm is the first constant-approximation no-substitution clustering algorithm that makes no structural assumptions on the input data set. The only limitation is the assumption that the input order is random. This assumption is useful in cases where there is no adversary and any order is as likely, as in the example applications mentioned above. Moreover, a random order is a standard assumption in many streaming settings, satisfied also whenever the input stream is drawn i.i.d.~from a distribution. In some cases, the randomness assumption might hold only approximately. In these cases, we expect a graceful degradation of the guarantees. We leave for future work a comprehensive treatment of input orders that are neither random nor adversarial. We expect the techniques of the current work to serve as a cornerstone in such a treatment.

\paragraph{Paper structure.} We discuss related work in \secref{related_works}. The formal setting and notations are defined in   \secref{setting}. The algorithm and its guarantees are given in \secref{algorithms}. In \secref{analysis}, we outline the proof of the main result. Some parts of the proof are deferred to appendices, which can be found in the supplementary material. We conclude with a discussion in \secref{discussion}.
\section{Related Work} \label{sec:related_works}
Several works have studied settings related to the no-substitution $k$-median clustering setting. \tabref{comparison} summarizes the upper bounds mentioned below.
First, some works studied related settings under an adversarial arrival order.

\cite{liberty2016algorithm} studied online $k$-means clustering, in which centers are sequentially selected, and each observed
point is allocated to one of the previously selected centers. 
The proposed algorithm can be applied to the no-substitution
setting, yielding an approximation factor of $O(\log(n))$, where $n$ is the stream length. For this algorithm to select a sublinear number of centers, the aspect ratio of the input data must be bounded.
\cite{bhaskara2020robust} studied the same setting as \cite{liberty2016algorithm}, improving the approximation factor to a constant, under the same assumptions. 
 \cite{bhattacharjee2020no} explicitly studied the no-substitution setting, and  provided an approximation factor of $O(k^3)$, under a different assumption on the input data set. 

 \newcommand{\adv}{adversarial}
 \begin{table}[t]
     \caption{Comparing our guarantees to previous works. $n$ is the stream length. Abbreviations: LSS16: \cite{liberty2016algorithm}, BR20:  \cite{bhaskara2020robust}, BM20: \cite{bhattacharjee2020no} Mo21: \cite{moshkovitz2021unexpected}, HS20: \cite{hess2020sequential}.}
  \label{tab:comparison}
   \begin{center}
  \begin{tabular}{lllll}
    Reference & arrival  & assumptions & approximation & number    \\
    & order & & factor & of centers  \\
   \toprule
    LSS16 & \adv & bounded aspect ratio & $O(\log (n))$ & $O(k\log^2(n))$  \\
    \midrule
    BR20 & \adv & bounded aspect ratio & constant & $O(k\log^2(n))$ \\
    \midrule
    BM20 & \adv & data properties & $O(k^3)$ & $O(k \log(k) \log (n))$ \\
    \midrule
    Mo21 &  random & none &  exponential in $k$ & $O(k^5)$ or $O(k \log (n))$ \\
    \midrule
    HS20 & random & bounded diameter &  constant$+$additive & $k$ \\
    \midrule
    This work & random & none & constant & $O(k \log^2(k))$\\
     \toprule                                                                   
  \end{tabular}
\end{center}
\end{table}

The setting of a random arrival order has been studied in several recent works.
\cite{hess2020sequential} proposed an algorithm that selects exactly $k$ centers. They obtained a constant approximation factor, with an additional additive term that vanishes for large streams, under the assumption that the metric space has a bounded diameter. Their guarantee is with high probability.
\cite{moshkovitz2021unexpected} obtained an approximation factor that is exponential in $k$, with a number of centers that is linear in $k$ and logarithmic in the stream length $n$. 
Assuming a known stream length, the same work also obtained an approximation factor which is exponential in $k$ while selecting $O(k^5)$ centers. Both guarantees hold only with a constant probability.  As can be seen in \tabref{comparison}, to date, the only known constant approximation algorithms for no-substitution $k$-median clustering impose structural assumptions on the input data.

A related sequential clustering setting is the streaming setting \cite[e.g.,][]{guha2000clustering,ailon2009streaming, chen2009coresets, ackermann2012streamkm++,braverman2016new}, in which the main restriction is the amount of memory available to the algorithm. This setting allows substituting selected centers, but algorithms in this setting can be used in the no-substitution setting, by collecting all the centers ever selected. However, we are not aware of any algorithm in this setting with a competitive bound on the total number of selected centers.

\section{Setting and Notation} \label{sec:setting}

For an integer $i$, denote $[i] := \{1,\ldots,i\}$.
Let $(X,\rho)$ be a finite metric space, where $X$ is a set of size $n$ and $\rho:X \times X \rightarrow \reals_+$ is a metric. For a point $x \in X$ and a set $T \subseteq X$, we denote $\rho(x,T) := \min_{y \in T} \rho(x,y)$. 
For an integer $k  \geq 2$, a $k$-clustering of $X$ is a set of (at most) $k$ points from $X$ which represent cluster centers. Throughout this work, whenever an item is selected based on minimizing $\rho$, we assume that ties are broken based on a fixed arbitrary ordering. Given a set $S \subseteq X$, the \emph{$k$-median risk} of $T$ on $S$ is  $R(S,T):= \sum_{x \in S} \rho(x, T)$. The \emph{$k$-median clustering problem} aims to select a $k$-clustering $T$ of $X$ with a minimal overall risk $R(X,T)$. We denote by $\OPT$ an optimal solution to this problem: $\OPT \in \argmin_{T  \subseteq X, |T| \leq k}R(X,T)$. 
 In the \emph{sequential no-substitution} $k$-median setting that we study, $X$ is not known a priori. The points from $X$ are presented to the algorithm sequentially, in a random order. We assume that the stream length, $n$, is provided as input to the algorithm; see \secref{discussion} for a discussion on supporting an unknown stream length. The algorithm may select an observed point as a center only before observing the next point. Any selected point cannot later be removed or substituted. The goal of the algorithm is to select a small number of centers, such that with a high probability, the overall risk of the selected set on the entire $X$ approximates the optimal $k$-median risk, $R(X,\OPT)$.
 An \emph{offline $k$-median algorithm} $\cA$ is an algorithm that takes as input a finite set of points $S$ and the parameter $k$, and outputs a $k$-clustering of $S$, denoted $\cA(S,k)$. 
For $\beta \geq 1$, we say that $\cA$ is a \emph{$\beta$-approximation} offline $k$-median algorithm on $(X,\rho)$, if for all input sets $S \subseteq X$, $R(S,\cA(S,k)) \leq \beta \cdot R(S, \OPT_S)$, where $\OPT_S$ is an optimal solution with centers from $S$. Formally, $\OPT_S \in \argmin_{T  \subseteq S, |T| \leq k}R(S,T)$.

\section{The algorithm}
\label{sec:algorithms}
In this section, we describe the proposed algorithm, \finalalgname.
\finalalgname\ receives as input the value of $k$, the confidence level $\delta \in (0,1)$ and the total stream length $n$. We further assume access to some black-box offline $k$-median algorithm $\cA$. The guarantees of \finalalgname\ depend on the approximation factor guaranteed by $\cA$. \finalalgname\ has only sequential access to the input stream.
Our main result is the following theorem, showing that \finalalgname\ obtains a constant approximation factor with a quasi-linear number of centers. A proof sketch of the theorem is given in \secref{analysis}. The full proof is provided in the supplementary.

\begin{theorem}\label{thm:final-algorithm}
  Let $k,n \in \nats$. Let $\delta \in (0,1)$. Let $(X,\rho)$ be a metric space of size $n$. Let $\Tout$ be the set of centers selected by \finalalgname\ for the input parameters $\delta,k,n$. Suppose that the input stream is  a random permutation of $X$, and that the input black-box algorithm $\cA$ is a $\beta$-approximation offline $k$-median algorithm. Then, with a probability at least $1-\delta$,

\begin{enumerate}
\item $|\Tout| \leq O\big(k \log^2(k/\delta)\big)$, and
\item $R(X,\Tout) \leq C\beta \cdot R(X,\OPT)$, where $C > 0$ is a universal constant. 
\end{enumerate}
\end{theorem}

We first give, in \secref{algoverview}, a short overview of the structure of  \finalalgname. A core element of the algorithm is the internal procedure \algname, which is used by \finalalgname\ with several different sets of input parameters, creating a multiscale effect that we discuss below.
We explain the details of \algname\ in \secref{desc_basic_alg}. Then, in \secref{desc_final_alg}, we explain the design of \finalalgname.

\thealgorithms

\subsection{Algorithm structure overview}\label{sec:algoverview}

\finalalgname\ is listed in \algref{final_alg}.
It works as follows: It initializes $\Theta(\log(k/\delta))$ copies of the procedure \algname, each with a different set of input parameters, where we define $\phia:= 150\log (32k/\delta)/\alpha$ for $\alpha > 0$. It then iteratively reads the points from the input stream one by one, and feeds each point to each of the copies of \algname, so that each copy sequentially observes the entire input stream. 
Whenever any copy of \algname\ selects a point as a center, this point is selected by \finalalgname\ as a center. 

The internal procedure \algname\ uses a robust risk estimation approach to select a small number of centers. The procedure gets as input several technical parameters, in addition to the input parameters of \finalalgname. We explain the procedure and the meaning of the technical parameters in \secref{desc_basic_alg}. 
The values of the technical parameters set by \finalalgname\ for each of the copies of \algname\ ensure that \finalalgname\ select a small number of centers, while obtaining a constant approximation factor.  This is done using a multiscale approach. We discuss this in detail in \secref{desc_final_alg}.

\subsection{The internal procedure: \algname} \label{sec:desc_basic_alg}
In this section, we present the internal procedure, \algname. It has been observed in previous work \citep{liberty2016algorithm} that knowing in advance the optimal risk on the entire stream allows a successful center selection. 
Under a random arrival order, one may hope that obtaining a good estimate for the optimal risk would be straight-forward. However, since the metric space is unbounded, even a small number of outliers can bias the risk estimate considerably. We overcome this challenge by showing that it suffices to estimate a version of the optimal risk that ignores outliers, and that this version can be estimated reliably from a small prefix of the stream.  Our analysis is based on a distinction between small and large optimal clusters, which we formally define in \secref{analysis} below.

\algname, listed in \algref{basic_algorithm}, uses the following notation.
For two sets $S,T \subseteq X$ and an integer $r$, denote by $\far_r(S,T) \subseteq S$ the set of $r$ points in $S$ that are the furthest from $T$ according to the metric $\rho$. If $|S| < r$, we define $\far_r(S,T) := S$ and call $\far_r(S,T)$ a \emph{trivial far set}. Denote by $R_r(S,T):=R(S\setminus \far_r(S,T),T)$ the risk of $T$ on $S$ after discounting the $r$ points that incur the most risk. 

Let $\ka:= k + 38\log(32k/\delta)$. 
\algname\ receives the same input parameters as \finalalgname. In addition, it requires several technical parameters, denoted $\alpha,\gamma,\iflarge$ and $\tr$. 
We explain the meaning of these parameters in the proper context below.

\algname~works in three phases. For $i \in [3]$, we denote by $P_i$ the set of points that are read in phase $i$. In each of the first two phases, an $\alpha$ fraction of the points in the stream is read. In the last phase, a $\gamma$ fraction of the points in the stream is read. Note that depending on the values of $\alpha$ and $\gamma$, the procedure may ignore a suffix of the stream. 
  The first two phases are
used for calculations. Centers are selected only during the
third phase.

In the first phase, a reference clustering
$\ca$ is calculated using the input offline algorithm $\cA$. Note that the centers in
$\ca$ cannot be selected as centers by \algname, since $\cA$
calculates $\ca$ only after observing all the points in the phase. In the second phase, an estimate $\eo$ of the risk of $\ca$ is calculated.  In the third phase, \algname\ observes points from
the stream one by one, deciding for each one whether it should
be selected as a center.  For each observed point, \algname\ first finds the center $c_i \in \ca$ which is closest to it. A point is then selected as a center if it satisfies one of three conditions: (1) its distance from $c_i$ is more than $\eo/(k \tr)$ (where $\tr$ is one of the input technical parameters); (2) $c_i$ does not yet have $\iflarge$ associated points (where $\iflarge$ is one of the technical parameters); or (3) no point close to $c_i$ has been selected
so far, as maintained by the Boolean variable $\undert_i$. 
In \thmref{gen-first-algorithm} in \secref{analysis}, we provide a guarantee on the output of \algname, which upper bounds the number of centers selected by \algname\ and the risk obtained by these centers, as a function of the technical parameters.

\textbf{Challenges and solutions in \algname.} We give here an informal explanation of the workings of \algname. The full analysis is provided in \secref{analysis} and the supplementary. Consider the clusters induced by some fixed optimal
$k$-median clustering on $X$. Call these clusters \emph{optimal
clusters}. Optimal clusters that are large (that is, include a sufficiently
large fraction of $X$) are easy to identify from a small random subset of
points. Therefore, approximate centers for these clusters will be identified in
the first phase of \algname\ by the black box algorithm $\cA$, and will be
included in $\ca$. Thus, the risk of $\ca$ on the subset of points that belong to large optimal clusters will be close to optimal. Note that $\ca$ is a clustering with $\ka > k$ centers. This overcomes the fact that the data set might include outliers which are not proportionally represented in each phase. Searching for a $k$-clustering in $P_1$ might thus lead to a clustering that is too biased towards such outliers. By increasing the size of the solution searched in the first phase to $\ka$, this allows the solution to include the outliers as centers, while still choosing also centers that are close to the optimal centers of all large clusters. 

The selection conditions of the third phase further guarantee that at least $M$ points are selected near each center $c_i \in \ca$ which represents a large cluster, thus making sure that a center which is very close to each such $c_i$ is selected. This allows bounding the obtained risk on points in large optimal clusters.

An additional challenge is to ensure that points in small optimal clusters are not too far from the set $\Tout$ of all selected centers. Since the metric space is unbounded, even a single such point can destroy the constant approximation factor. To overcome this, \algname\ selects a point near each center in $\ca$, as well as all the points that are far from $\ca$. The threshold $\eo/(k\tr)$, which defines a point as far, is set so that the number of points selected from large optimal clusters can be bounded. This bound crucially relies on the accuracy of $\eo$ as an estimate for the risk of $\ca$ on large clusters. The required accuracy is obtained by ignoring the points that are furthest from $\ca$ when calculating the risk on $P_2$ (see line \ref{line-P2}). For small clusters, they hold a small number of points by definition, thus the number of such points selected as centers is also bounded.

\subsection{The clustering algorithm: \finalalgname} \label{sec:desc_final_alg}

As described above, \finalalgname\ runs several copies of \algname\ on the same input stream: each point is read from the input stream and then fed to each of the copies of \algname.
Each of these copies selects some centers, and the set of all selected centers is the solution of \finalalgname. The difference between the copies is in the value of the technical input parameters. First, the value of $\alpha$ (the stream fraction for the first and second phases) is progressively doubled, starting with $\alpha_1 \equiv \delta/(4k)$ and ending with $\alpha_{I+1} \equiv \alpha_1 \cdot 2^I$, where $I$ is set so that $\alpha_{I+1} \in (1/12,1/6]$. The value of $\tr$ (which controls the selection threshold) is set to $\phi_{\alpha_{I+1}}$ in all copies, based on the largest value of $\alpha$. In all but the last copy, $\gamma$ (the stream fraction for the third phase) is set to $2\alpha_i$, where $i$ is the copy index. This means that the first and second phase of copy $i$ are each of size $\alpha_i$, while the third phase is of size $2\alpha_i$ (the rest of the stream is ignored by copy $i$). Therefore, copy $i$ observes a $4\alpha_i= 2\alpha_{i+1}$ fraction of the stream. Thus, the first two phases of copy $i+1$ exactly overlap with the fraction of the stream observed by copy $i$. \figref{examplefinalalg} illustrates the overlap of phases in consecutive copies of \algname. In the last copy, indexed by $I+1$, $\gamma$ is set to $1-2\alpha_{I+1}$, thus this copy reads the entire stream. It can be seen that each point, except for the first $2n\alpha_1$ points, participates in the third phase of exactly one copy of \algname. 
As our analysis below shows, this overlap between the phases guarantees that the set of centers selected by the copies of \algname\ obtains a small risk on all the small optimal clusters. The last copy is slightly different: in addition to setting the length of the third phase $\gamma$ to include all the points of the stream, it also sets the technical parameter $\iflarge$ to a number larger than $1$. We show below that in this way, the selected set of centers obtains a small risk on large optimal clusters as well.

\begin{figure}[b]
  \centering
  \includegraphics[width = 0.9\textwidth 
  ]{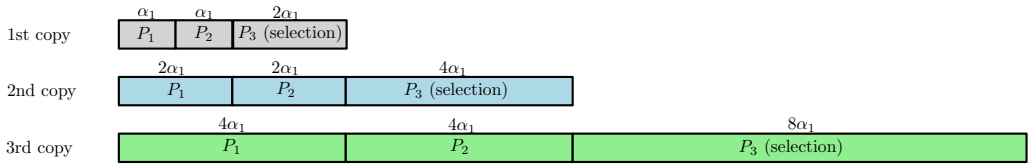}
  \caption{Illustrating the phases in copies $1,2,3$ of \algname\
    within \finalalgname. }
  \label{fig:examplefinalalg}
\end{figure}

The use of several copies of \algname\ with different phase sizes ensures that \finalalgname\ selects only a quasi-linear number of centers, by overcoming the following challenge: On the one hand, the first phase and the second phase must be small, otherwise we might miss a point that needs to be selected, since no points are selected in these two phases. On the other hand, small first and second phases lead to a poor quality of the reference clustering $\ca$. The multiscale approach makes sure that almost all points participate in some selection phase, while at the same time improving the quality of $\ca$ as $\alpha$ grows.
Our analysis below shows that in this way, the number of centers selected by each copy remains similar, even though larger values of $\alpha$ lead to a larger selection phase.

\section{Analysis} \label{sec:analysis}

In this section, we give an overview of the proof of \thmref{final-algorithm}. The supplementary provides the full proof.
Denote the centers in the optimal solution $\OPT$ by $\{c_i^*\}_{i\in k}$, and the clusters induced by $\OPT$ (the \emph{optimal clusters})  by $\{C_i^*\}_{i \in k}$. Formally, $C_i^*:= \{ x \in X \mid i = \argmin_{j \in [k]}\rho(c_j^*,x) \}.$ 
 Optimal clusters with fewer than $\phia$ points are \emph{$\alpha$-small} optimal clusters, and the complement are \emph{$\alpha$-large}.
Denote the indices of $\alpha$-small and $\alpha$-large clusters by $\smallc :=\{ i \in [k] \mid |C_i^*| < \phia \}$ and $\largec := [k] \setminus \smallc $, and the points in these clusters by $\smallp:= \bigcup_{i \in \smallc} C_i^*$ and $\largep:= \bigcup_{i \in \largec} C_i^*$. 
To prove \thmref{final-algorithm}, we provide
 the following guarantee on the centers selected by \algname.
\begin{theorem} \label{thm:gen-first-algorithm}
Let $k,n \in \nats$. Let $\delta \in (0,1)$. Let $(X,\rho)$ be a metric space of size $n$. Let $\Tout$ be the set of centers selected by \algname\ for the input parameters $\delta,k,n, \alpha \in (0,1/6]$, $\gamma \in (\alpha,1-2\alpha]$, $\iflarge \in \nats$, $\tr > 0$. Suppose that the input stream is a random permutation of $X$, and that $\cA$ is a $\beta$-approximation offline $k$-median algorithm. Then, with a probability at least $1-\delta/2$, 
\begin{enumerate}
\item $|\Tout| =O\left(\frac{\gamma}{\alpha}k\log(k/\delta) +k\tau+(k+\log(k/\delta))\iflarge\right)$; \label{gen-numcenters}
\item \label{gen-small} For any $i \in [k]$, if $c_i^* \in P_3$, then
\begin{equation}\label{eq:riskclust}
  R(C_i^*,\Tout) \leq R(C_i^*,\{c_i^*\}) + \frac{|C_i^*|}{k\tr} (36\beta+20) R(X,\OPT);
\end{equation}

\item \label{gen-large} If  $\gamma=1-2\alpha$, $\tr=\phia$  and $\iflarge=\log (8\ka/\delta)$, then
\begin{equation}
R(\largep,\Tout)\leq R(\largep,\OPT) + (468\beta+260) R(X,\OPT).\label{eq:first-algo-large}
\end{equation}
\end{enumerate}
\end{theorem} 
The proof of \thmref{final-algorithm},  provided in \appref{proof-main-thms}, uses part \ref{gen-numcenters} to bound the total number of centers, and parts \ref{gen-small} and \ref{gen-large} to bound the risk of small and large optimal clusters, respectively. 
The statement of \thmref{gen-first-algorithm} hints to the reason for the multiscale design of \finalalgname: On the one hand, part \ref{gen-small} gives an approximation upper bound only if the cluster center appears in the third phase. For this to hold with a high probability, the third phase needs to be very large. On the other hand, for a quasi-linear bound on selected centers in part \ref{gen-numcenters}, the ratio between the third phase and the first phases ($\gamma/\alpha$) needs to be small. The multiscale design overcomes this by keeping the ratio constant in each copy of \algname, while ensuring that the union of all third phases is almost the entire stream.

In the rest of the section, we give the proof of \thmref{gen-first-algorithm}: In \secref{events_definitions}, we introduce the concept of \emph{linear bin divisions}, and derive its properties. In
\secref{basic-algo-proof}, we give an overview of the proof of \thmref{gen-first-algorithm}, using the results of \secref{events_definitions}. The full proof is provided in the supplementary.

\subsection{Linear bin divisions} \label{sec:events_definitions}

A main tool in the proof of \thmref{gen-first-algorithm} is partitioning sets of points from the input stream into subsets, such that each of the subsets is probably well represented in a relevant random subset of the stream. \cite{meyerson2004k} defined the concept of a \emph{bin division} of $X$ with respect to a clustering $T$, which is a partition of $X$ into bins of equal size, where points are allocated to bins based on their distance from $T$. Here, we define the concept of a \emph{linear bin division}, in which the bins linearly increase in size.
This gradual increase allows keeping the size of the first bin independent of the stream length, while still proving that with a high probability, the overlap of each of the bins in the division with random subsets of the stream is close to expected. The fixed size of the first bin is crucial for deriving guarantees that are independent of the stream length. In addition, the ratio between adjacent bin sizes is kept bounded, which allows proving a bounded approximation factor. 

\begin{definition}[Linear bin division] \label{geo-div-bins}
  Let $W,T \subseteq X$ be finite sets, and $z \in \nats$. A  \emph{$z$-linear bin division of $W$ with respect to $T$} is a partition $\bins \equiv (\bins(1), \ldots, \bins(L))$ of $W$ (for an integer $L$) such that: 
  \begin{enumerate}
\item If $z \leq |W|$, then $\forall i \in [L], |\bins(i)| \geq  z\cdot(i+1)/2$. Otherwise, the bin division is called \emph{trivial}, and defined as $\bins := \bins(1) := W$. \label{bin-linear} 
\item  $|\bins(1)|\leq \frac{5}{2}z$.   \label{bin-b1} 
\item $\forall i \in [L-1], |\bins(i+1)|/|\bins(i)| \leq 3/2$. \label{bin-ratio}    
\item  $ \forall i \in [L-1]$, and $\forall x \in \bins(i), x' \in \bins(i+1)$, it holds that $\rho(x,T) \geq \rho(x',T)$. \label{bin-order}  
 \end{enumerate}
\end{definition} 
A linear bin division exists for any size of $W$: For $|W| \leq z$, the conditions trivially hold. For $|W| \geq z$, 
the three first properties hold for the following allocation of bin sizes: Let $L$ be the largest integer such that $B := \sum_{i \in [L]} z\cdot (i+1)/2 \leq |W|$. Set the size of $\bins(i)$ to $z\cdot(i+1)/2 + (|W|-B)/L$. Property \ref{bin-linear} clearly holds. Property \ref{bin-b1} holds since $(|W|-B)/L \leq (z(L+2)/2)/L \leq 3z/2$. Property \ref{bin-ratio} holds since $(i+2+a)/(i+1+a) \leq 3/2$ for all $i \geq 1$ and any non-negative $a$.  
To satisfy property \ref{bin-order}, allocate the elements of $W$ into the bins in descending order of their distance from $T$.

We say that a set is \emph{well-represented} in another set if the size of its overlap with the set is similar to expected. This extends naturally to bin divisions. Formally, this is defined as follows.

\begin{definition}[Well-represented]
  Let $W$ be a finite set and $A,B \subseteq W$. We say that $B$ is \emph{well-represented in $A$ for $W$} if 
  $|B \cap A|/|B| \in [r/2, (3/2)r]$, where $r := |A|/|W|$. 
  We say that a linear bin division $\bins$ of $W$ is well represented in $A$ for $W$ if each bin in $\bins$ is well represented in $A$ for $W$. 
\end{definition}

The following lemma shows that if $A$ is selected uniformly at random from $W$, then any fixed set $B$ is well-represented in $A$ with a high probability. Moreover, the same holds for any $z$-linear bin division of $W$ with a sufficiently large $z$. The proof of the lemma is provided in \appref{wellproofs}.

\begin{lemma}\label{lem:combined-well-represented}
  Let $W$ be a finite set. Let $B \subseteq W$ be a set, and let $\bins$ be a $z$-linear bin division of some subset of $W$ with respect to some $T$, for some integer $z$. Let $A \subseteq W$ be a set of size $r|W|$ selected uniformly at random from $W$. Then the following hold:
  \begin{enumerate}
  \item  With a probability at least $1-2e^{-\frac{r|B|}{10}}$, $B$ is well-represented in $A$ for $W$.
    \item If
  $z \geq 10 \log(4/\delta)/r$, then with a probability  $1-\delta$, $\bins$ is well represented in $A$ for $W$. 
  \end{enumerate}
\end{lemma}

We now state and prove a main property of linear bin divisions, which will allow us to use sub-streams of the input stream to bound the risk of a clustering on $X$.

\begin{lemma}\label{lem:upperboundOPT} 
Let $W\subseteq X$ and let $z \in \nats$. Let $\bins$ be a $z$-linear bin division of  $W$ with respect to some $T$. Let $A \subseteq W$ and $r \in (0,1)$. If  $\forall i \in [L],|\bins(i) \cap A|/|\bins(i)| \leq r$, then 
$
R(A \setminus \bins(1),T) \leq \frac{3}{2} r R(W,T).
$
\end{lemma}

\begin{proof}
Let $L$ be the number of bins in $\bins$.
We first prove the following inequality, which relates the risk of the intersection of a bin with $A$ to the risk of the preceding bin. 

\begin{equation}\label{eq:bini}
\forall i \in [L-1], \, \,  R(A \cap \bins(i+1) ,T ) \leq  \frac{3}{2}r R(\bins(i),T). 
\end{equation}
To prove \eqref{bini}, fix $i \in [L-1]$, and denote $b:= \max_{x \in \bins(i+1)}\rho(x,T).$  By the assumptions, we have \mbox{$|A \cap \bins(i+1)| \leq r|\bins(i+1)|$.} Hence,
$R(A \cap \bins(i+1),T ) \leq r|\bins(i+1)| \cdot b.$
By property~\ref{bin-ratio} of linear bin divisions, $|\bins(i+1)| \leq \frac{3}{2}|\bins(i)|.$ Therefore, $R(A \cap \bins(i+1),T ) \leq \frac{3}{2}r|\bins(i)|\cdot b$. 
In addition, by property \ref{bin-order} of linear bin divisions  and the definition of $b$, 
$
\forall x \in \bins(i), b \leq \rho(x,T).
$
Therefore,  $R(\bins(i),T) \geq |\bins(i)|\cdot b$. 
Combining the two inequalities, we get \eqref{bini}.
It follows that:
\begin{align*}
R(A \setminus \bins(1),T ) &= \sum_{i=2}^L R(A \cap \bins(i), T )  \leq  \frac{3}{2} r \sum_{i=1}^{L-1}  R(\bins(i),T) 
                              \leq \frac{3}{2} r R(W,T).
\end{align*}
This proves the statement of the lemma. \end{proof}

To prove \thmref{gen-first-algorithm}, we define an event, denoted $\event$, which holds with a high probability. We prove each of the claims in the theorem under the assumption that this event holds. To define the event, we first provide some necessary notation. Consider a run of \algname\ with some fixed set of input parameters, and assume that the input stream is a random permutation of the points in $X$. 
Recall that $\ca = \{c_1,\ldots,c_\ka\}$ is the clustering calculated in the first phase of \algname. Denote the clusters induced by $\ca$ on $X$ by $C_1,\ldots,C_{\ka}$, where $C_i:=\{x \in X \mid i= \argmin_{j \in \ka} \rho(c_j,x)  \}$. We define several sets and bin divisions, which will be used to define the required event.  
Let $\binsa$ be a $(\phia/15)$-linear bin division of $X$ with respect to $\OPT$. Define $\fu:=\far_{k\phia}(X \setminus P_1,\ca)$, and let $\binsb$ be a $(\phia/3)$-linear bin division of $X \setminus (P_1 \cup \fu)$ with respect to $\ca$. Define $\fl:=\far_{4(k+1) \phia}(X \setminus P_1,\ca)$, and let $\binsc$ be a $(\phia/3)$-linear bin division of $X \setminus (P_1 \cup \fl)$ with respect to $\ca$.  
Lastly, define $\fs:=\far_{5(k+1) \phia}(X \setminus P_1,\ca)$.  Note that each of these objects may be trivial if the stream is small. 

The event $\event$ is defined as the following conjunction:
(1)  For each $i \in \largec$, $C_i^*$ is well-represented in $P_1$ and in $P_1 \cup P_2$ for $X$. \manuallabel{event-optimal-clusters}{1}
(2) Each of $\fu$ and $\fl$ is trivial or well-represented in $P_2$ for $X \setminus P_1$; $\fs$ is trivial or well-represented in $P_3$ for $X \setminus P_1$. \manuallabel{event-well-represent-fars}{2}
(3)  $\binsa$ is trivial or well-represented in $P_1$ for $X$; Each of $\binsb$ and $\binsc$ is trivial or well-represented in $P_2$ for $X \setminus P_1$, and \manuallabel{event-well-represent-bins}{3}
(4)   
For each $i \in [\ka]$, one of the first  $\log_2 (8\ka/\delta)$ points observed from $C_i$ in $P_3$ is closer to $c_i$ than at least half of the points in $C_i \cap P_3$. \manuallabel{event-phase3}{4}

The following lemma, proved in \appref{event}, shows that $\event$  holds with a high probability.  

\newcommand{\nmin}{n_{\min}}

\begin{lemma}\label{lem:allevents}
Consider a run of \algname\ with fixed input parameters. Assume that the input stream is a random permutation of the points of $X$. Then $\event$ holds with a probability at least $1-\delta/2$. 
\end{lemma}

\subsection{\thmref{gen-first-algorithm}: proof overview} \label{sec:basic-algo-proof}

In this section,
we give an overview of the proof of \thmref{gen-first-algorithm}. The full proof 
is given in \appref{mainlemmas}. 
In the first phase of \algname, the clustering $\ca$ is calculated using the offline algorithm $\cA$. We first bound the risk of $\ca$ on points in $\alpha$-large optimal clusters. 
  \begin{lemma} \label{lem:upper-bound} 
If  $\event$ holds, then $R(\largep ,\ca) \leq (18\beta+10) R(X,\OPT)$. 
\end{lemma}

In the second phase of \algname, the goal is to estimate a truncated version of the risk of $\ca$ on $X$. This truncated version ignores the risk on the outliers, which are the $k\phia$ furthest points from $\ca$ in $X$. Setting the estimate to $\eo \equiv \frac{1}{3\alpha}R_{2\alpha (k+1) \phia}(P_2,\ca)$ ignores the furthest $2\alpha (k+1) \phia$ points in $P_2$, which with a high probability include all the $k\phia$ outliers in $X$. 

The following lemma bounds $\eo$ between two versions of a truncated risk of $\ca$. 
\begin{lemma} \label{lem:Phase2:eo-max} 
If $\event$ holds, then
$\frac{1}{9} R_{5  (k+1) \phia}( X \setminus P_1, \ca) \leq \eo \leq  R_{k\phia}(X,\ca)$.
\end{lemma}
The proof of this lemma uses the linear bin divisions studied in \secref{events_definitions}. In particular, the upper bound uses \lemref{upperboundOPT}, and the lower bound uses the specific construction of the bin sizes.

From the lemmas above, it can be seen that $\eo$ is also upper bounded by the optimal risk, as follows. $\alpha$-small optimal clusters are smaller than $\phia$, thus their total size is less than $k\phia$. It follows that
\begin{equation}\label{eq:phase1-k-phia-to-la}
R_{k \phia}(X,\ca)  \leq R(X\setminus \smallp,\ca) \equiv R(\largep,\ca).
\end{equation}
Combining \lemref{upper-bound}, \eqref {phase1-k-phia-to-la} and  \lemref{Phase2:eo-max}, we get that $\eo$ is upper bounded by the optimal risk.

In the third phase, the points selected as centers include (among others) all the points with a distance of more than $\eo/(k\tr)$ from $\ca$. This allows bounding the risk of $\Tout$ on points in some of the optimal clusters by a linear expression in $\eo/(k\tr)$. Combined with the upper bound on $\eo$, this allows proving part \ref{gen-small} of \thmref{gen-first-algorithm}.
Part \ref{gen-large} bounds the risk of $\Tout$ on large optimal clusters, for specific settings of the technical parameters of $\algname$. To prove this part, we show that under these parameter settings, the centers selected based on the reference clustering $\ca$ include sufficiently many points around each large optimal cluster. We then show that one of these points is close to the optimal center.

Lastly, part \ref{gen-numcenters} of \thmref{gen-first-algorithm} upper bounds the number of centers selected by \algname. 
The following lemma upper-bounds the number of points that are selected because they are far from $\ca$. 
\begin{lemma}\label{lem:boundingN}
Let $N:=|\{ x \in P_3 \mid \rho(x,\ca) > \eo/(k\tr) \}|$. Under $\event$,
$N  \leq 8 (k+1) \phia \frac{\gamma}{1-\alpha}+9k\tr$.
\end{lemma}

The proof  \lemref{boundingN} uses the lower bound on $\psi$ shown in \lemref{Phase2:eo-max}, to conclude that there cannot be too many points that are more than $\psi/(k\tr)$ far from $\ca$, since this would make the truncated risk in the lower bound larger than $\eo$. We note that the resulting upper bound is independent of the stream size, due to the specific construction of the bin sizes in the linear bin division. Using this lemma, we can prove part \ref{gen-numcenters} of \algname. 
The full proof of \thmref{gen-first-algorithm} is given in \appref{mainlemmas}.

\section{Discussion} \label{sec:discussion}

We provided the first constant-approximation algorithm for sequential no-substitution $k$-median clustering that does not require structural assumptions on the input data. Moreover, the number of centers selected by \finalalgname\ is only quasi-linear in $k$.
\finalalgname\ can also be used in a bounded memory setting, by using a bounded-memory black-box streaming algorithm $\cA$ as input. This is because \finalalgname\ requires a memory size of only $\tilde{O}(k\log(1/\delta))$ on top of requirements of the black-box algorithm. 
\finalalgname\ can also be used when the stream length $n$ is not known in advance, using a simple doubling trick. \cite{moshkovitz2021unexpected} showed that an unknown stream length with an approximation factor that does not depend on $n$ require
a logarithmic dependence on $n$ in the number of selected centers. Indeed, a doubling trick for \finalalgname\ would preserve the approximation factor, increasing the number of selected centers only by the unavoidable $\Theta(\log(n))$ factor.

\bibliographystyle{plainnat}

\bibliography{Mybib}

\appendix
\clearpage

\section{Proof of \thmref{final-algorithm}} \label{ap:proof-main-thms}

In this section, we use \thmref{gen-first-algorithm} to prove \thmref{final-algorithm}. First, define the following event, which we denote by $G$: 
\[
  \forall i \in [k], c_i^*\in X\text{ is not observed in the first }n \cdot \delta/(2k)\text{ reads from the input stream.}
  \]
  Observe that $G$ holds with a probability at least $1-\delta/2$: For each $c_i^*$, the probability that it is observed in the prefix is $\delta/(2k)$, and a union bound over $k$ centers gives a probability of at most $\delta/2$ that the event does not hold. We use this observation in the proofs below. First, we prove the guarantees of \algname.

\begin{proof}[Proof of \thmref{final-algorithm}]
  First, we show that we can assume that \thmref{gen-first-algorithm} holds simultaneously for all copies of \algname\ run in \finalalgname. 
  The number of copies of \algname\ is $I+1$. 
    Since \algref{final_alg} calls \algname\ with the confidence parameter $\delta'$, each of the executions of \algname\ satisfies the claims of \thmref{gen-first-algorithm} with a probability at least $1-\delta'/2$. Since $\delta' \equiv \delta/(I+1)$, by a union bound, \thmref{gen-first-algorithm} holds for all executions simultaneously with a probability at least $1-\delta/2$. With a probability at least $1-\delta$, this holds simultaneously with the event $G$ defined above. We henceforth assume that all these events hold. 
    Note that since $\alpha_1 = \delta/(4k)$, we have
    \[
    I+1 \equiv \floor{\log_2(1/(6\alpha_1))}+1 = O(\log(k/\delta)). 
  \]
  Therefore, $\delta' =\Omega(\delta/\log(k/\delta))$ and $\log(k/\delta') = O(\log(k/\delta))$.
  
  To prove the first part of \thmref{final-algorithm}, which
  upper bounds  $|\Tout|$, we apply the first claim of
  \thmref{gen-first-algorithm} (with $\delta'$), which states that
  $|\Tout| =O(\frac{\gamma}{\alpha}k\log(k/\delta')
  +k\tau+(k+\log(k/\delta'))\iflarge)$, to each of the copies of \algname.
  In all the copies of \algname, we have
  \[
    \tau := \phi_\alpham = O(\log(k/ \delta'))=O(\log(k/ \delta)).
  \]
  In addition, in all but the last copy of \algname, we have
  $M := 1$ and $\gamma/\alpha := 2$. 
  Thus, the number of centers selected
  in each of these copies of \algname\ is $O(k\log(k/\delta))$. There are $I = O(\log(k/\delta))$ such copies, hence the total number of
  centers selected in all but the last copy is $O(k\log^2(k/\delta))$.  In the last
  copy, \algname\ is called with
  $M = O(\log(\ka/\delta')) = O(\log(k/\delta))$ and $\gamma =
  1-2\alpham$. Since $\alpham > 1/12$, we have $\gamma/\alpham =
  O(1)$. Therefore, again by the first claim of \thmref{gen-first-algorithm}, the
  number of centers selected by the last copy is
  $O(k\log(k/\delta)+\log^2(k/\delta))$. The overall number of centers selected in \finalalgname\ is thus
  $O(k\log^2(k/\delta))$.

We now prove the second part of \thmref{final-algorithm}, which bounds the approximation factor of $\Tout$ on $X$. Since $G$ holds, none of the centers $c_i^*$, for $i \in [k]$, appears in the first $2\alpha_1 n$ points read from the input stream. Now, each point observed after this prefix appears in the selection phase of some copy of \algname\ in \finalalgname. It follows that for each $i \in [k]$, \eqref{riskclust} in the second part of \thmref{gen-first-algorithm} holds for one of the copies.

We apply \eqref{riskclust} noting that in all copies, we have $\tau := \phi_{\alpham}$. In addition, all copies get the full stream $X$ as input. Moreover, the set of centers $\Tout$ selected by \finalalgname\ is a superset of each of the sets selected in each of the calls to \algname, thus its risk cannot be larger. It follows that for all $i \in [k]$,
\[
  R(C_i^*,\Tout) \leq R(C_i^*,\{c_i^*\}) + \frac{|C_i^*|}{k\phi_{\alpham}} (36\beta+20) R(X,\OPT).
  \] 
In particular, for $i \in \smallcmax$, we have $|C_i^*| \leq \phi_\alpham$.
Therefore,
\begin{align}
R(\smallpmax,\Tout) = \sum_{i \in \smallpmax} R(C_i^*,\Tout)\leq R(\smallpmax,\OPT) + (36\beta+20) R(X,\OPT) \label{eq:final-algo-small}.
\end{align} 

Now, consider the last copy of \algname. In this copy, the conditions of the third claim of \thmref{gen-first-algorithm} hold, and so \eqref{first-algo-large} holds. Combining \eqref{first-algo-large} with \eqref{final-algo-small}, and noting that $X = \smallpmax \cup \largepmax$, the second part of \thmref{final-algorithm} follows. 
\end{proof}

\section{Proof of \lemref{combined-well-represented}}\label{ap:wellproofs}
In this section, we prove \lemref{combined-well-represented}, stated in \secref{events_definitions}, which provides guarantees for the property of being well-represented. 
\begin{proof}[Proof of \lemref{combined-well-represented}]
  For the first claim, we use a multiplicative concentration bound of \citet{dubhashi1996balls}. This bound states that for negatively associated random variables $Z_1,\ldots,Z_n$ taking values in $\{0,1\}$, where $Z := \sum_{i\in[n]} Z_i$, it holds that
  \[
    \P[|Z - \E[Z]| \geq \epsilon \E[Z]] \leq 2e^{-\frac{\epsilon^2 \E[Z]}{2+\epsilon}}.
    \]
 Let $Z_i$ be equal to $1$ if the $i$'th element of $B$ is in $A$.
By \cite{joag1983negative}, the random variables of a uniform sample without replacement are negatively associated, hence the inequality above holds for  $\{Z_i\}$.
In addition, $Z = |B\cap A|$ and $\E[Z] = r|B|$. The first claim follows by setting  $\epsilon := 1/2$.

To prove the second claim, note that the first claim implies that each set $\bins(i)$ in $\bins$ is well-represented in $A$ for $W$ with a probability at least $1-2\exp(-r |\bins(i)|/10 )$. 
By property \ref{bin-linear} of linear bin divisions, and the assumption that $z \geq 10 \log(4/\delta)/r$, we have
\[
    |\bins(i)| \geq z(i+1)/2 \geq 5(i+1) \log(4/\delta)/r.
  \]
   Therefore, $\exp(-r |\bins(i)|/10) \leq  (\delta/4)^{(i+1)/2}.$

   By a union bound over all the bins in $\bins$, the probability that at least one bin in $\bins$ is not well-represented in $A$ for $W$ is upper-bounded by $2\sum_{i=1}^\infty (\delta/4)^{(i+1)/2} \leq \frac{\delta}{2-\sqrt{\delta}}\leq \delta$, which proves the claim.
\end{proof}

\section{Proof of \lemref{allevents}}\label{ap:event}

In this section, we prove \lemref{allevents}, which states that the event $E$ holds with high probability.
\begin{proof}[Proof of \lemref{allevents}] 
  First, we show that part \ref{event-optimal-clusters} of $\event$ holds with a probability at least $1-\delta/32$. This is proved by applying the first part of \lemref{combined-well-represented} to each of the clusters $C_i^*$ for $i \in \largec$, and noting that their size is at least $\phia$. Applying the lemma with the sets $P_1,X$ and with the sets $P_1 \cup P_2, X$ and using a union bound over at most $k$ large optimal clusters, it follows that the probability that part \ref{event-optimal-clusters} does not hold is at most  $4k\exp(-\alpha\phia/10) = 4k\exp(15\log(\delta/(32k)) = 4k\cdot(\delta/(32k))^{15} \leq \delta/32$.

   For part \ref{event-well-represent-fars}, we also show that it holds with a probability at least $1-\delta/32$. we similarly apply the first part of \lemref{combined-well-represented} to each of the required sets using the ratios $|P_2|/|X\setminus P_1| = \alpha/(1-\alpha) \geq \alpha$ and $|P_3|/|X \setminus P_1| = \gamma/(1-\alpha) \geq \alpha$ (since $\gamma \geq \alpha$). Since $\fu,\fl,\fs$, if not trivial, are each of size at least $k\phia$, we get that the probability that any of these sets is not well represented is at most $6\exp(-\alpha \phia/10) \leq 6\exp(15\log(\delta/(32k)) \leq \delta/32$. 
 
     Next, we show that part \ref{event-well-represent-bins} holds with a probability at least $1-\delta/4$. For $\binsa$, if it is not trivial, then the conditions of the second part of \lemref{combined-well-represented} hold with confidence parameter $\delta/8$, since it is a $z$-linear bin division with $z= \phia /15 \geq 10\log(32 /\delta)/\alpha$, and $\alpha = r$, where $r$ is the ratio used in \lemref{combined-well-represented}.  For $\binsb$ and $\binsc$, if they are not trivial, the conditions of the second part of \lemref{combined-well-represented} with confidence parameter $\delta/16$ hold, since they are $z$-linear bin divisions with $z = \phia/3 = 50\log (32k/\delta)/\alpha \geq 10\log(4\cdot 16/\delta)/r$.

  Lastly, we show that part \ref{event-phase3} holds with a probability at least $1-\delta/8$.  Fix $i \in [\ka]$. If
  $|C_i \cap P_3| \leq \log_2 (8\ka/\delta)$, then the statement of
  part \ref{event-phase3} trivially holds. Assume that
  $|C_i \cap P_3| > \log_2 (8\ka/\delta)$. Points from
  $C_i \cap P_3$ are observed in a random order in the stream. The fraction of
  points in $C_1 \cap P_3$ that are closer to $c_i$ than at least half of the points in $C_1 \cap P_3$ is at least half. Therefore, every draw from $C_i \cap P_3$ has a probability of at least half of satisfying this requirement, conditioned on all previous points not satisfying this requirement. The probability that none of the points out of the first $\log_2 (8\ka/\delta)$ satisfies this requirement is thus at most $2^{-\lceil\log_2 (8\ka/\delta)\rceil}\leq \delta/(8\ka)$. A union bound over $i \in [\ka]$ gives the desired bound.

By a union bound over all the parts of the event, it hold with a probability at least $1-\delta/2$. 
\end{proof}

\section{Proof of \thmref{gen-first-algorithm}}\label{ap:mainlemmas}

In this section, we prove \thmref{gen-first-algorithm}.
First, we state two helpful lemmas. Then, we give the proof of the theorem. The proofs of the lemmas below and of the lemmas stated in \secref{basic-algo-proof}, which are all used in the proof of the theorem, are provided in the subsections below. 

The following two lemmas allow us to bound the risk of the set of selected centers $\Tout$ on points in all optimal clusters. The first lemma considers optimal clusters whose center is observed in the third phase. 

\begin{lemma} \label{lem:small_clusters_analysis} 
For any $i \in [k]$, if $c_i^* \in P_3$ then $R(C_i^*,\Tout) \leq R(C_i^*,\{c_i^*\}) + 2|C_i^*|\eo/(k\tr).$
\end{lemma}
The second lemma bounds the risk of the set of selected centers on large optimal clusters.
\begin{lemma} \label{lem:large_clusters_analysis} 
If $\event$ holds, $\gamma=1-2\alpha$, $\tr=\phia$ and $\iflarge=\log (8\ka/\delta)$, then 
$$
R(\largep,\Tout)\leq R(\largep,\OPT) + 26 R(\largep,\ca).
$$
\end{lemma}
Based on the lemmas above and the lemmas stated in \secref{basic-algo-proof}, the proof of \thmref{gen-first-algorithm} can be given.
\begin{proof}[Proof of \thmref{gen-first-algorithm}]
By \lemref{allevents}, $\event$ holds with probability $1-\delta/2$. Assume that $\event$ holds.
First, we upper bound $|\Tout|$. Note that \algname\ selects a point only if one of the three conditions in line \ref{line-nv-thr} hold.

By \lemref{boundingN}, the number of points selected by \algname\ due to the first condition is
\[
  N \leq 8 (k+1) \phia \frac{\gamma}{1-\alpha}+9k\tr.
\]
Since  $\phia = O(\log (k/\delta)/\alpha)$ and $\alpha \leq 1/6$, we have \mbox{$N = O((\frac{\gamma}{\alpha}k\log(k/\delta) +k\tau)$}. The rest of the points selected by \algname\ are those satisfying one of the other conditions of line \ref{line-nv-thr}. This allows at most $M$ additional points for each \mbox{$i\in [\ka]$}. Thus, the total number of points selected by \algname\ is \mbox{$O(\frac{\gamma}{\alpha}k\log(k/\delta) +k\tau+\ka\cdot \iflarge)$}. Since by definition, $\ka= O(k + \log(k/\delta))$, part \ref{gen-numcenters} of \thmref{gen-first-algorithm} immediately follows.

Part \ref{gen-small} of \thmref{gen-first-algorithm} follows immediately, by combining \lemref{small_clusters_analysis}, \eqref {phase1-k-phia-to-la}, \lemref{Phase2:eo-max} and \lemref{upper-bound}.  Part \ref{gen-large} is also immediate, by combining \lemref{upper-bound} and \lemref{large_clusters_analysis}.  This completes the proof of the theorem.
\end{proof}

In the next subsections, we prove the lemmas used in the proof of \thmref{gen-first-algorithm}.
\subsection{Bounding the risk of $\ca$ on large clusters} \label{sec:p1-a}

In this section, we prove \lemref{upper-bound}, which bounds
$R(\largep ,\ca)$.

For a given $x\in X$, denote the center closest to it in
$\ca$ by
$c_{\alpha}(x)$. Recall that $\binsa$ is a $(\phia/15)$-linear bin division of $X$ with respect to $\OPT$. 
To prove  \lemref{upper-bound}, we first bound the risk of $\ca$ on large clusters, using an expression that depends on the optimal risk and on the points in the first phase of \algname.
\begin{lemma}\label{lem:upper-bound1}
  If $\event$ holds and $n \geq \phia /15$, then
  \[
    R(\largep ,\ca) \leq R(\largep,\OPT) + \frac{4}{\alpha}(R(P_1 \setminus \binsa(1),\OPT)+R(P_1  ,\ca)).
    \]
\end{lemma}

\begin{proof}
  For any $x \in X$ and any $i \in [k]$, we have
  \[
    \rho(x,c_{\alpha}(x)) \leq \rho(x,c_{\alpha}(c_i^*)) \leq \rho(x,c_i^*)+\rho(c_i^*,c_{\alpha}(c_i^*)).
    \]
Therefore,  
  \begin{align}
R(\largep ,\ca)&= \sum_{i \in \largec} \sum_{x \in C_i^* } \rho(x,c_{\alpha}(x)) 
\leq \sum_{i \in \largec} \sum_{x \in C_i^* } \rho(x,c_i^*) +\sum_{i \in \largec} \sum_{x \in C_i^* }\rho(c_i^*,c_{\alpha}(c_i^*)) \notag  
 \\
 &\leq R(\largep ,\OPT) + \sum_{i \in \largec} |C_i^*| \cdot \rho(c_i^*,c_{\alpha}(c_i^*)). \label{eq:clargeopt}
  \end{align}
 
We now bound $\rho(c_i^*,c_{\alpha}(c_i^*))$. Define $A_i:=C_i^* \cap P_1 \setminus \binsa(1)$ (we show below that $A_i \neq \emptyset$). From the definition of $c_\alpha$, we have 
 \[
\forall x \in A_i, \quad \rho(c_i^*,c_{\alpha}(c_i^*))\leq \rho(c_i^*,c_{\alpha}(x)) \leq \rho(c_i^*,x)+ \rho(x,c_{\alpha}(x)).
\]
Hence,
\[
  \rho(c_i^*,c_{\alpha}(c_i^*)) \leq  \frac{1}{|A_i|}\sum_{x \in A_i} (\rho(c_i^*,x)+ \rho(x,c_{\alpha}(x))) = \frac{1}{|A_i|}(R(A_i,\{c_i^*\})+R(A_i,\ca)).
  \]

Denote $l_i := |C_i^*|$ and $l_i':= |A_i|$. Then, by \eqref{clargeopt},
\begin{align}
 R(\largep ,\ca) \leq R(\largep,\OPT) + \sum_{i \in \largec}\frac{l_i}{l_i'}(R(A_i,\{c_i^*\})+R(A_i,\ca)).\label{eq:lratio}
\end{align}
We now upper bound $l_i/l_i'$. Since $n \geq \phia/15$, we have that $\binsa(1)$ is non-trivial. Since $\event$ holds, $\binsa(1)$ and $\{C_i^*\}_{i \in \largec}$ are well-represented in $P_1$ for $X$. In addition, $|P_1| = \alpha |X|$. 
 Hence,
\begin{align}\label{eq:liprime}
  l_i' \equiv |C_i^* \cap P_1 \setminus \binsa(1)| \geq |C_i^* \cap P_1| -|\binsa(1) \cap P_1| \geq \alpha|C_i^*|/2-(3/2)\alpha |\binsa(1)|.
\end{align}
By the definition of $\alpha$-large optimal clusters, for all $i \in \largec, \phia \leq |C_*^i| \equiv l_i$. Since $\binsa$ is a $(\phia/15)$-linear bin division of $X$, by property \ref{bin-b1} of linear bin divisions, we have $|\binsa(1)|\leq \frac{5}{2}\cdot \frac{\phia}{15}=\phia/6 \leq l_i/6$. Hence, by \eqref{liprime},
$l_i'\geq  \alpha l_i/2 -(3/2)\alpha \cdot (l_i/6) \geq \alpha l_i/4.$
Therefore, $l_i/l_i' \leq 4/\alpha$. Note that this also implies that $A_i$ is non-empty. Thus, from \eqref{lratio}, we have 
 \begin{align*}
   R(\largep ,\ca) 
   \leq R(\largep, \OPT) + \frac{4}{\alpha}\Big(&R(P_1 \cap \largep \setminus \binsa(1),\OPT)\\
   &+R(P_1 \cap \largep \setminus \binsa(1),\ca)\Big).
\end{align*}
Since $P_1 \cap \largep \setminus \binsa(1) \subseteq P_1 \cap  \binsa(1)$, this leads to the statement of the lemma. 
\end{proof}
In the next lemma, we bound $R(P_1,\ca)$, which appears in the RHS of the inequality given in \lemref{upper-bound1}. 
In the proof of this lemma, the importance of selecting $\ka$ centers for $\ca$ instead of only $k$ becomes apparent, as it allows deriving an upper bound on the risk of $\ca$ that disregards the outliers in $\bins(1)$.

\begin{lemma} \label{lem:claim-phase1}
If $\event$ holds and $n \geq \phia/15$, then $R(P_1,\ca) \leq 2\beta R(P_1 \setminus \binsa(1),\OPT)$.
\end{lemma}

\begin{proof} 
  Denote the optimal $\ka$-median solution for $P_1$ with centers from the entire $X$ by
  \[
    \OPT^\ka_{P_1} := \argmin_{T  \subseteq X, |T| \leq \ka}R(P_1,T).
  \]
  Denote the optimal $\ka$-median solution for $P_1$ with centers from $P_1$ by
  \[
    \widetilde{\OPT}^\ka_{P_1} := \argmin_{T  \subseteq P_1, |T| \leq \ka}R(P_1,T).
  \]

  It is well known \cite[see, e.g.,][]{guha2000clustering} that for any set $P_1$ in any metric space $(X, \rho)$, the risk ratio between the solution that uses only centers from $P_1$ and the solution that uses the entire metric space is bounded by $2$: $R(P_1,\widetilde{\OPT}^\ka_{P_1})\leq 2R(P_1,\OPT^{\ka}_{P_1}) $.
  Now, the black-box algorithm $\cA$ which outputs $\ca$ in the first phase of \algname\ is a $\beta$-approximation offline algorithm. Therefore,  we have
  \begin{equation}\label{eq:beta}
    R(P_1,\ca) \leq \beta  R(P_1,\widetilde{\OPT}^{\ka}_{P_1}) \leq 2\beta R(P_1,\OPT^{\ka}_{P_1}).
  \end{equation}

We bound $R(P_1,\OPT^{\ka}_{P_1})$ by the risk of the clustering defined by $\OPT$, with additional centers that are the furthest from $\OPT$ in $P_1$.
Recall that $\binsa$ is a $(\phia/15)$-linear bin division of $X$ with respect to $\OPT$. 
Define $A := \binsa(1)  \cap P_1$. By property \ref{bin-order} of linear bin divisions, the points in $A$ are the furthest from $\OPT$ in $P_1$. We now bound the size of $A$. 
Since $\event$ holds and $n \geq \phia/15$, $\binsa(1)$ is non-trivial and well represented in $P_1$ for $X$. Therefore, $|A| \leq \frac{3}{2}\alpha|\binsa(1)|$.  In addition, by property \ref{bin-b1} of linear bin divisions, $|\binsa(1)|\leq \frac{5}{2}\cdot\frac{\phia}{15}\leq  \phia/6$. Therefore,
$|A| \leq \alpha \phia/4 \leq 38\log (32k/\delta).$ 

It follows that $|\OPT \cup A| \leq k+38\log (32k/\delta) \equiv \ka$. Thus, by the definition of $\OPT^{\ka}_{P_1}$, 
\begin{align*}
  R(P_1,\OPT^{\ka}_{P_1}) &\leq R(P_1,\OPT \cup A) =R(P_1 \setminus A,\OPT \cup A) \\
                          &= R(P_1 \setminus \binsa(1),\OPT \cup A) \leq  R(P_1 \setminus \binsa(1),\OPT).
\end{align*}
Combining the inequality above with \eqref{beta}, we get the statement of the lemma.
\end{proof}

\lemref{upper-bound1} and \lemref{claim-phase1} can be combined to get
\begin{equation}\label{eq:interbound}
 R(\largep ,\ca) \leq R(\largep,\OPT) + \frac{4}{\alpha}(2\beta+1)R(P_1 \setminus \binsa(1),\OPT).
\end{equation}

The proof of \lemref{upper-bound} is now almost immediate.
\begin{proof}[Proof of \lemref{upper-bound}]
  We distinguish between two cases based on the stream length $n$.
  First, consider the case $n<\phia / 15$. In this case, all optimal clusters are smaller than $\phia/15$, thus they are all $\alpha$-small optimal clusters. Hence, $\largep =\emptyset$, and so $R(\largep,\ca)=0 $, which trivially satisfies the required upper bound.

  Now, suppose that $n \geq \phia/15$. Thus, $\binsa$ is non-trivial. Since $\event$ holds, it follows that $\binsa$ is well-represented in $P_1$ for $X$. Hence, $\forall i \in [L], |P_1 \cap \bins(i)| \leq \frac{3}{2}\alpha |\bins(i)|$. We now apply \lemref{upperboundOPT} with $W:=X, T := \OPT$, $A:=P_1$, $\bins:=\binsa$, and $r:=\frac{3}{2}\alpha$, we have
  \[
    R(P_1 \setminus \binsa(1),\OPT) \leq \frac{9}{4}\alpha R(X,\OPT).
    \]
Combining this with \eqref{interbound}, we get
\begin{align*}
R(\largep ,\ca) 
&\leq   R(X,\OPT) + 9(2\beta+1)R(X,\OPT) \leq (18\beta+10)R(X,\OPT), 
\end{align*}
As claimed.
\end{proof}

\subsection{Upper-bounding the risk estimate $\eo$}\label{sec:p2-a}
In \lemref{upper-bound}, we showed that the risk of $\ca$ on large optimal clusters is bounded by a constant factor over the optimal risk.
In the second phase, \algname\ calculates an estimate of this risk, defined as $\eo := \frac{1}{3\alpha}R_{2\alpha (k+1) \phia}(P_2,\ca)$.
In this section, we prove the upper bound in \lemref{Phase2:eo-max}, which shows that the removal of $2\alpha (k+1) \phia$ outliers from the risk estimate indeed guarantees that $\eo$ is upper bounded by the risk of $\ca$ on large clusters. The proof of the lower bound in \lemref{Phase2:eo-max} is provided in \secref{bounding-number-of-far-points}.

 \begin{proof}[Proof of \lemref{Phase2:eo-max} (Upper bound)]
  We distinguish between two cases, based on the size of the stream.
  If $n < 2(k+1)\phia$, then 
  $|P_2|< 2\alpha(k+1)\phia$. 
  In this case, $R_{2\alpha (k+1) \phia}(P_2,\ca) = 0$, thus $\eo = 0$ and the upper bound trivially holds.
  
  Now, consider the case $n \geq 2(k+1)\phia$. In this case, $|X \setminus P_1| \geq (1-\alpha)n \geq (k+1)\phia$. Recall that $\fu \equiv \far_{k\phia}(X \setminus P_1,\ca)$, and  $\binsb$ is a $(\phia/3)$-linear bin division of $(X \setminus P_1) \setminus \fu$ with respect to $\ca$. $\fu$ is non-trivial, since $|X \setminus P_1| \geq k\phia$. Similarly, $\binsb$ is non-trivial, since $|(X \setminus P_1) \setminus \fu|\geq (k+1)\phia - k\phia \geq \phia/3$.
  
  As shown in 
  \secref{basic-algo-proof}, $R_{k\phia}(X,\ca) \leq R(\largep,\ca)$.  Hence, to prove the lemma, it suffices to prove that 
$\eo \leq R_{k\phia}(X,\ca).$ In the rest of the proof we derive this inequality.

We first show that $(\fu \cup \binsb(1))\cap P_2 \subseteq \far_{2\alpha (k+1) \phia}(P_2, \ca)$. This would imply that the points in $\fu \cup \binsb(1)$ do not contribute to the risk $R_{2\alpha (k+1) \phia}(P_2,\ca)$, which is used in the calculation of $\eo$. 
Since $\event$ holds, $\fu $ and $ \binsb(1)$ (which are both non-trivial) are well-represented in $P_2$ for $X \setminus P_1$. We also have $|P_2|/|X \setminus P_1| = \alpha/(1-\alpha) \leq \frac{4}{3}\alpha$, where the last inequality follows since $\alpha \leq 1/6$. Therefore, $|\fu \cap P_2| \leq \frac{3}{2}\cdot \frac{4}{3}\alpha|\fu| = 2\alpha k \phia$, and similarly, $|\binsb(1) \cap P_2| \leq 2\alpha|\binsb(1)| \leq 2\alpha \phia$, where the last inequality follows from property \ref{bin-b1} of linear bin divisions, since $(5/2)\cdot(\phia/3)\leq \phia$. It follows that 
\[
  |(\fu \cup \binsb(1)) \cap P_2| \leq |\fu \cap P_2| + |\binsb(1) \cap P_2|\leq2\alpha (k+1)\phia. 
\] 
By the definitions of $\fu$ and $\binsb$, and by property \ref{bin-order} of linear bin division, we have that the points in $\fu \cup \binsb(1)$ are the furthest points from $\ca$ in $X \setminus P_1$. Thus, $(\fu \cup \binsb(1))\cap P_2$ are the furthest points from $\ca$ in $P_2$.  
Therefore, $(\fu \cup \binsb(1)) \cap P_2 \subseteq \far_{2\alpha (k+1) \phia}(P_2, \ca)$, as we wished to show.
This implies that 
\begin{equation} \label{eq:p2A11}
R_{2\alpha (k+1) \phia}(P_2,\ca) \leq R(P_2 \setminus (\fu \cup \binsb(1)),\ca).
\end{equation} 

Now, since $\binsb$ is well-represented in $P_2$ for $X \setminus P_1$, denoting the number of bins in $\binsb$ by $L$, we have that for all $i\in [L], |\binsb(i) \cap P_2|/|\binsb(i)| \leq 2\alpha$. Since $\binsb$ is a partition of $(X \setminus P_1) \setminus \fu$, we have that $P_2 \cap \binsb(i) = (P_2 \setminus \fu) \cap \binsb(i)$. 
Therefore, we can apply \lemref{upperboundOPT} with $W:=(X\setminus P_1) \setminus \fu$, $\bins:=\binsb$, $T:=\ca$, $A:=P_2 \setminus \fu$, and $r:=2\alpha$. It follows that 
 \begin{equation} \label{eq:p2A13}
 R((P_2 \setminus \fu) \setminus \binsb(1),\ca) \leq 3\alpha R((X\setminus P_1) \setminus \fu,\ca) = 3\alpha R_{k\phia}(X \setminus P_1, \ca).
\end{equation}
The last equality follows from the definition of $\fu$. By combining \eqref{p2A11} and \eqref{p2A13}, we get that $R_{2\alpha (k+1) \phia}(P_2,\ca)  \leq 3 \alpha R_{k\phia}(X \setminus P_1,\ca).$
Plugging in the definition of $\eo$, this implies that $\eo \leq R_{k\phia}(X \setminus P_1,\ca) \leq R_{k\phia}(X,\ca)$, as required. 
 \end{proof}

 \subsection{Bounding the number of selected far points}\label{sec:bounding-number-of-far-points}

In this section, we prove \lemref{boundingN}, which bounds the number of points that are selected as centers and their distance from $\ca$ is more than $\eo/(k\tr)$. We first prove  an auxiliary lemma, and provide the proof for the lower bound in \lemref{Phase2:eo-max}, which is used in the proof of \lemref{boundingN}. The auxiliary lemma provides a relationship between the risks of consecutive bins.

 \begin{lemma}\label{lem:claim1} 
Let $L$ be the number of bins in $\binsc$. If $\event$ holds, then
   \begin{align*}
  \forall i \in [L-1], \quad R(\binsc(i) \cap P_2,\ca) &\geq \frac{\alpha}{3}R(\binsc(i+1),\ca).
   \end{align*}
 \end{lemma}
 \begin{proof}
Fix $i \in [L-1]$, and define $b := \max_{x \in \binsc(i+1)}\rho(x,\ca)$. By property \ref{bin-order} of linear bin divisions, $\rho(x,\ca) \geq b$ for all $x \in \binsc(i)$. 
Therefore,
\begin{align*}
  &R(\binsc(i) \cap P_2,\ca)\geq |\binsc(i) \cap P_2| \cdot b,  \quad\text{ and }\\
  &R(\binsc(i+1),\ca)\leq |\binsc(i+1)| \cdot b.
\end{align*}
It follows that
\[
  \frac{R(\binsc(i) \cap P_2,\ca)}{R(\binsc(i+1) ,\ca)} \geq \frac{|\binsc(i) \cap P_2|}{|\binsc(i+1)|}.
  \]
  Thus, to prove the claim, it suffices to lower-bound the RHS by $\alpha/3$.
  Since $\event$ holds, $\binsc$ is well-represented in $P_2$ for $X \setminus P_1$. Since $|P_2|/|X \setminus P_1| = \alpha/(1-\alpha) \geq \alpha$, it follows that  \mbox{$|\binsc(i) \cap P_2| \geq (\alpha/2)|\binsc(i)|$.} 
  By property \ref{bin-ratio} of linear bin divisions, $|\binsc(i)| \geq \frac{2}{3}|\binsc(i+1)|.$ Therefore, $|\binsc(i) \cap P_2| \geq (\alpha/3)|\binsc(i+1)|$, which proves the claim.
\end{proof}
The next step is to prove the lower bound in \lemref{Phase2:eo-max}, which bounds from below the risk estimate $\eo$ calculated in \algname. Recall that $\eo := \frac{1}{3\alpha}R_{2\alpha (k+1) \phia}(P_2,\ca)$.

 \begin{proof}[Proof of \lemref{Phase2:eo-max} (Lower bound)]
First, note that if $|X \setminus P_1|\leq 5  (k+1) \phia$, then $R_{5  (k+1) \phia}( X \setminus P_1, \ca)=0$, and so the lemma trivially holds. We henceforth assume that $|X \setminus P_1| > 5  (k+1) \phia$.   Recall that $\fl \equiv \far_{4(k+1)\phia}(X \setminus P_1,\ca)$, and  $\binsc$ is a $(\phia/3)$-linear bin division of $(X \setminus P_1) \setminus \fl$ with respect to $\ca$. Under the assumption, $|X \setminus P_1| > 4(k+1)\phia$, thus $\fl$ is non-trivial. Similarly, $\binsc$ is non-trivial, since $|(X \setminus P_1) \setminus \fl|\geq 5(k+1)\phia - 4(k+1)\phia > \phia/3$.

To prove the lemma, we show  the following inequality:
\begin{align} \label{eq:p2l0}
R_{2\alpha (k+1) \phia}(P_2,\ca)\geq \frac{\alpha}{3} R_{5  (k+1) \phia}( X \setminus P_1, \ca).
\end{align}
This would immediately imply that $\eo \equiv \frac{1}{3\alpha}R_{2\alpha (k+1) \phia}(P_2,\ca) \geq \frac{1}{9} R_{5  (k+1) \phia}( X \setminus P_1, \ca)$, as required.
Denote by $L$ the number of bins in $\binsc$. Define
\newcommand{\tX}{\widetilde{X}}
\[
  \tX:=\bigcup_{i=2}^L \binsc(i) \equiv (X \setminus P_1) \setminus (\fl \cup \binsc(1)).
\]
See \figref{lower_bound_fig} for an illustration of the sets used in this proof. To prove \eqref{p2l0}, we show the following string of inequalities:
\begin{equation}\label{eq:string}
  R_{2\alpha (k+1) \phia}(P_2,\ca)\geq R(P_2 \setminus \fl,\ca) \geq \frac{\alpha}{3} R(\tX,\ca) \geq \frac{\alpha}{3} R_{5  (k+1) \phia}( X \setminus P_1, \ca).
\end{equation}

\begin{figure}[H]
  \centering
  \includegraphics[width = 0.8\textwidth]{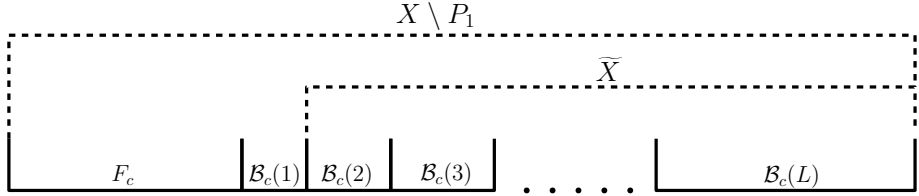}
  \caption{The sets used in the proof of the lower bound in \lemref{Phase2:eo-max}. The horizontal axis represents the distance of the points in the set from $\ca$, in decreasing order. } 
  \label{fig:lower_bound_fig}
\end{figure} 
First, we prove the first inequality in \eqref{string} by removing far points from $P_2$. 
Recall that $\fl:=\far_{4(k+1) \phia}(X \setminus P_1,\ca)$ are the furthest points from $\ca$ in $X \setminus P_1$. Since $\event$ holds, $\fl$ is well-represented in $P_2$ for $X \setminus P_1$. Since $|P_2|/|X \setminus P_1| = \alpha/(1-\alpha) \geq \alpha$, this implies that
\[
|P_2 \cap \fl|  \geq \frac{\alpha}{2} |\fl| = 2 (k+1) \phia \alpha.
\]
Hence,  the $2(k+1)\phia \alpha$ furthest points from $\ca$ in $P_2$ are in $\fl$. Formally, $\far_{2(k+1)\phia \alpha}(P_2,\ca) \subseteq \fl$. Therefore,
$R_{2(k+1)\phia \alpha}(P_2,\ca) \geq R(P_2 \setminus \fl,\ca)$,
which proves the first inequality of \eqref{string}. 

Next, we prove the last inequality of \eqref{string}. 
By property \ref{bin-b1} of linear bin divisions, we have $|\binsc(1)| \leq \frac{5}{2}\cdot \frac{\phia}{3} \leq \phia$. Therefore,
\[
  |\fl \cup \binsc(1)| \leq  |\fl|+|\binsc(1)| \leq 4(k+1) \phia + \phia 
  \leq 5(k+1)\phia.
\]
In addition, by the definition of $\fl$ and property \ref{bin-order} of linear bin divisions,  the points in $\fl \cup \binsc(1)$ are the furthest from $\ca$ out of the points in $X \setminus P_1$. Hence, $\fl \cup \binsc(1) \subseteq \far_{5(k+1)\phia}(X \setminus P_1,\ca)$. 
It follows that $\tX \supseteq (X \setminus P_1) \setminus \far_{5(k+1)\phia}(X \setminus P_1,\ca)$.
Therefore, \mbox{$R(\tX,\ca) \geq R_{5(k+1)\phia}(X \setminus P_1,\ca)$}, which proves the last inequality in \eqref{string}. 

To complete the proof, we have left to prove the second inequality of \eqref{string}, which states that $R(P_2 \setminus \fl,\ca) \geq \frac{\alpha}{3}R(\tX,\ca).$ 
We have 
\begin{align}
R(P_2 \setminus \fl,\ca)=\sum_{i \in [L]} R( \binsc(i) \cap (P_2 \setminus \fl),\ca)=\sum_{i \in [L]} R( \binsc(i) \cap P_2 ,\ca).
\end{align}
The last equality follows since $\bins$ is a partition of $(X \setminus P_1) \setminus \fl$, thus $\bins(i) \cap \fl = \emptyset$.
By \lemref{claim1}, since $\event$ holds, $\forall i \in [L-1], R( \binsc(i) \cap P_2 ,\ca)\geq \frac{\alpha}{3}R(\binsc(i+1) ,\ca)$. 
It follows that 
\begin{align}
R(P_2 \setminus \fl,\ca)& \geq  \sum_{i \in [L-1]} R( \binsc(i) \cap P_2 ,\ca) \geq
 \frac{\alpha}{3}\sum_{i=2}^L  R(\binsc(i) ,\ca) =\frac{\alpha}{3} R(\tX,\ca). \notag 
\end{align}
The last equality follows from the definition of $\tX$.
This completes the proof of the second inequality of \eqref{string}, thus proving the lemma.
\end{proof}

Based on the lemmas above, \lemref{boundingN} can now be proved.
\begin{proof}[Proof of \lemref{boundingN}] 
  \newcommand{\overth}{\mathcal{T}}
  Denote the set of points in $P_3$ that exceed the distance threshold used in \algname\ by
  \[
    \overth := \{ x \in P_3 \mid \rho(x,\ca) > \eo/(k\tr) \}.
  \]
  Then $N = |\overth|$. Recall that $\fs := \far_{5(k+1) \phia}(X \setminus P_1, \ca)$. We have $N \leq |\overth \setminus F_d| + |P_3 \cap \fs|$.
    To upper bound $|\overth \setminus F_d|$, note that from the definition of $\fs$ and $\overth$, 
    \[
      R_{5  (k+1) \phia}(X \setminus P_1, \ca) = R( (X \setminus P_1) \setminus \fs, \ca) \geq R(\overth \setminus \fs, \ca) \geq |\overth \setminus F_d|\cdot \eo/(k\tr).
      \]
      On the other hand, by \lemref{Phase2:eo-max}, $\eo \geq \frac{1}{9} R_{5  (k+1) \phia}( X \setminus P_1, \ca)$. It follows that $|\overth \setminus F_d| \leq 9k\tr$.

      To upper bound $|P_3 \cap \fs|$, consider first the case where $\fs$ is trivial. In this case, $|X \setminus P_1| < 5(k+1)\phia$, thus
      \[
        |P_3 \cap \fs| \leq  |P_3| = |X \setminus P_1| \cdot (|P_3|/|X \setminus P_1|) \leq 5(k+1)\phia \cdot \frac{\gamma}{1-\alpha}.
        \]
        If $\fs$ is non-trivial, then since $\event$ holds, $\fs$ is well-represented in $P_3$ for $X \setminus P_1 $. Since we have  \mbox{$|P_3|/ |X \setminus P_1 | = \frac{\gamma}{1-\alpha}$}, it follows that
\[
  |P_3 \cap \fs| \leq \frac{3}{2}\frac{\gamma}{1-\alpha}|F_d| = \frac{3}{2}\frac{\gamma}{1-\alpha} \cdot 5 (k+1) \phia < 8 (k+1) \phia \frac{\gamma}{1-\alpha}.
  \]
 Therefore, in both cases, $|P_3 \cap \fs| \leq 8 (k+1) \phia \frac{\gamma}{1-\alpha}$.
Combining this with the upper bound on $|\overth \setminus \fs|$, the statement of the lemma follows. 
\end{proof}

\subsection{Risk approximation for small and large optimal clusters} \label{sec:small_clusters_analysis}

In this section, we prove \lemref{small_clusters_analysis} and \lemref{large_clusters_analysis}. The former lemma upper bounds the risk of the selected centers on optimal clusters whose center appears in the third stage, and is  used for bounding the risk of points in small optimal clusters. The latter lemma upper bounds
the risk of the selected centers on large optimal clusters, under an appropriate setting of algorithm parameters.

First, we provide an auxiliary lemma, which will be used to prove both of these lemmas. It shows that all points appearing in the last phase are close to some center in the set of selected centers. 
 \begin{lemma} \label{lem:all_x_risk_bounded}
 For all $x \in P_3,$ $\rho(x,\Tout)\leq 2\eo/(k \tr)$.
 \end{lemma}
 \begin{proof}
 If $x$ is selected by \algref{basic_algorithm}, then $x \in \Tout$, so $\rho(x, \Tout) =0$. Thus, in this case, the lemma trivially holds. 
 Now, suppose that $x$ is not selected as a center. 
Recall that $\ca = \{c_1,\ldots,c_{\ka}\}$, and denote $i := \argmin_{i \in [\ka]}\rho(x, c_i)$. Due to the selection conditions on line   \ref{line-nv-thr} of \algref{basic_algorithm}, we have that, since $x$ was not selected,  $\rho(x,c_i) \leq \eo/(k \tr)$ and $\undert_i = \texttt{TRUE}$ when $x$ is observed.  
From the latter, it follows that some $y \in C_i$ was selected such that $\rho(c_i,y)\leq \eo/(k \tr)$. 
Therefore, $\rho(x,\Tout) \leq \rho(x,y) \leq \rho(x,c_i)+\rho(c_i,y) \leq 2\eo/(k \tr)$. 
 \end{proof}

The proof of \lemref{small_clusters_analysis} is now almost immediate.
\begin{proof}[Proof of \lemref{small_clusters_analysis}]
Fix $i \in [k]$. 
Observe that by the triangle inequality,
\begin{equation*}
  R(C_i^*,\Tout)  \leq R(C_i^*,\{c_i^*\}) + |C_i^*|\rho(c_i^*,\Tout).
\end{equation*}
By the assumption of the lemma, $c_i^* \in P_3$. Applying \lemref{all_x_risk_bounded} to $x:=c_i^*$, we get that $\rho(c_i^*,\Tout) \leq 2\eo/(k \tr)$. Combined with the inequality above, this proves the lemma.
\end{proof}

We now turn to prove \lemref{large_clusters_analysis}, which upper bounds the risk of the selected centers on large optimal clusters. 
First, we relate the risk of points from $\alpha$-large optimal clusters in the third phase to the risk of the clustering $\ca$ calculated in the first phase.
\begin{lemma} \label{lem:T'approxca} 
If $\event$ holds, $\tr=\phia$ and $\iflarge=\log (8\ka/\delta)$, then $R(\largep \cap P_3,\Tout) \leq 13 R(\largep,\ca).$
\end{lemma}

\begin{proof}
  \newcommand{\light}{\texttt{Lt}}
  \newcommand{\heavy}{\texttt{Hv}}
Recall that $C_1,\ldots,C_{\ka}$ are the clusters induced by the centers in $\ca$. 
We distinguish between \emph{light clusters} and \emph{heavy clusters}. Light clusters are those that include many points from small optimal clusters in the third phase. Formally, the set of light clusters is defined as
\[
  \light :=\{i \in [\ka] \mid |C_i \cap P_3 \cap \smallp|/|C_i\cap P_3| \geq 1/4\}.
\]
Heavy clusters are non-light clusters, denoted $\heavy:= [\ka] \setminus \light$.
We have
\begin{equation}\label{eq:both}
  R(\largep\cap P_3, \Tout) \leq \sum_{i \in \light}R(C_i \cap P_3, \Tout) + \sum_{i \in \heavy}R(C_i \cap P_3 \cap \largep, \Tout).
\end{equation}
We bound each of these terms separately. 

For light clusters, we show that these clusters are relatively small, and bound the risk contributed by each point in the cluster. The total number of points in light clusters is 
\begin{align} \label{eq:boundlight}
  \sum_{i \in \light}|C_i \cap P_3| \leq 4\sum_{i \in \light}|C_i \cap P_3 \cap \smallp| \leq 4 | \smallp| \leq 4k\phia.
\end{align}
Moreover, for each $x \in C_i \cap P_3$ for $i \in \light$, we have by \lemref{all_x_risk_bounded} that $\rho(x,\Tout) \leq 2\eo/(k\tr)$. By the assumptions of the lemma, $\tr:=\phia$. Hence, combining with \eqref{boundlight}, we get
\begin{align}\label{eq:light}
\sum_{i \in \light}R(C_i \cap P_3,\Tout) \leq   \frac{2\eo}{k\phia}  \sum_{i \in \light}|C_i \cap P_3| \leq 8 \eo \leq 8R(\largep,\ca).
\end{align} 
The last inequality follows from \lemref{Phase2:eo-max}.

For heavy clusters, let $i \in \heavy$. Denote $r_i:=\rho(c_i,\Tout)$.
By the triangle inequality,
\begin{align*}
R(C_i \cap P_3 \cap \largep,\Tout) &\leq R(C_i \cap P_3 \cap \largep,\{c_i\})+|C_i \cap P_3 \cap \largep| \cdot r_i.
\end{align*}
Let $R_i := \{ x \in C_i \cap P_3 \mid \rho(x,c_i) \geq r_i\}$.
To upper bound $r_i$, note that
\[
  R(C_i \cap P_3 \cap \largep,\{c_i\}) \geq |R_i\cap \largep| \cdot r_i.
\]
Therefore,
\[
  r_i \leq \frac{R(C_i \cap P_3 \cap \largep,\{c_i\})}{|R_i\cap \largep|}.
  \]
It follows that
\begin{equation}\label{eq:boundri}
  R(C_i \cap P_3 \cap \largep,\Tout) \leq R(C_i \cap P_3 \cap \largep,\{c_i\})\cdot \left(1+ \frac{|C_i \cap P_3 \cap \largep|}{|R_i \cap \largep|}\right).
\end{equation}

We now lower-bound  $|R_i \cap \largep|$. \algref{basic_algorithm} selects the first $M = \log (8\ka/\delta)$ points from cluster $C_i$ observed in the third phase. From part \ref{event-phase3} of the event $\event$, we have that at least one of those points, call it $x^*$, is closer to $c_i$ than at least half of the points in $C_i \cap P_3$. Formally, $\rho(c_i, x^*) \leq \rho(c_i, x)$ for at least half of the points $x \in C_i \cap P_3$. Since $x^* \in \Tout$, we have $r_i \leq \rho(c_i, x^*)$, and so $r_i \leq \rho(c_i, x)$ for at least half of the points $x \in C_i \cap P_3$. It follows that $|R_i| \geq |C_i \cap P_3|/2$. 

On the other hand, since $i \in \heavy$, we have
$|R_i \cap \smallp| \leq |C_i \cap P_3 \cap \smallp| \leq |C_i \cap P_3|/4.$
Therefore, $|R_i \cap \largep| = |R_i| - |R_i \cap \smallp| \geq |C_i \cap P_3|/4 \geq |C_i \cap P_3 \cap \largep|/4$.
Combined with \eqref{boundri}, we get 
\[
  \forall i \in \heavy, \quad R(C_i \cap P_3 \cap \largep,\Tout) \leq 5R(C_i \cap P_3 \cap \largep,c_i).
\]
It follows that $\sum_{i \in \heavy} R(C_i \cap P_3 \cap \largep,\Tout) \leq 5R(\largep,\ca)$. Combining this with \eqref{both} and \eqref{light}, we get the statement of the lemma. 
\end{proof}

Based on the lemma above, the proof of \lemref{large_clusters_analysis} can be now provided. 
\begin{proof}[Proof of \lemref{large_clusters_analysis}]

To bound $R(\largep,\Tout)$, we bound $R(\largep \cap P_3,\Tout)$ and $R(\largep \cap (P_1 \cup P_2),\Tout)$ separately. The former is bounded using \lemref{T'approxca}, which gives $R(\largep \cap P_3,\Tout) \leq 13 R(\largep,\ca).$

To bound $R(\largep \cap (P_1 \cup P_2),\Tout)$, we first show that at least half of the points in each $\alpha$-large optimal cluster $C_i^*$ are in $P_3$. 
Fix $i \in \largec$. Since $\event$ holds, $C_i^*$ is well-represented in $P_1 \cup P_2$ for $X$. Since $|P_1 \cup P_2|/ |X| =2\alpha$ and $\alpha \leq 1/6$, it follows that $|C_i^* \cap (P_1 \cup P_2)|/|C_i^*|\leq \frac{3}{2} \cdot 2\alpha \leq 3\alpha  \leq 1/2$. Since $\gamma=1-2\alpha$, we have $P_1 \cup P_2 \cup P_3 = X$, hence $|C_i^* \cap P_3|/|C_i^*| \geq 1/2$.
Now, define an injection $\mu:P_1 \cup P_2 \rightarrow P_3$ such that $x \in C_i^*$ if and only if $\mu(x) \in C_i^*$. Such an injection exists since most of the points in each $\alpha$-large optimal cluster are in $P_3$.
For each $i \in \largec$ and each $x \in C_i^* \cap (P_1 \cup P_2)$, we have
\begin{equation}\label{eq:mu}
  \rho(x, \Tout) \leq \rho(x,c^*_i) + \rho(c^*_i,\mu(x)) + \rho(\mu(x),\Tout).
\end{equation}
Now, since $\mu$ is an injection, we have that for any $T \subseteq X$,
\[
  R(\{\mu(x) \mid x \in C_i^* \cap (P_1 \cup P_2)\}, T) \leq R(C_i^* \cap P_3, T).
  \]
  Summing  \eqref{mu} over $x \in C_i^* \cap (P_1 \cup P_2)$, it follows that
  \begin{align*}
    R(C_i^* \cap (P_1 \cup P_2), \Tout) &\leq R(C_i^* \cap (P_1 \cup P_2), \{c^*_i\}) + R(C_i^* \cap P_3, \{c^*_i\}) + R(C_i^* \cap P_3, \Tout)\\
    &= R(C_i^*, \OPT) + R(C_i^* \cap P_3, \Tout).
  \end{align*}
  Summing over $i \in \largec$, we get
  \begin{align*}
    R(\largep \cap (P_1 \cup P_2), \Tout) \leq R(\largep, \OPT) + R(\largep \cap P_3, \Tout).
  \end{align*}
  It follows that 
  \begin{align*}
    R(\largep, \Tout) &\leq R(\largep \cap (P_1 \cup P_2), \Tout) + R(\largep \cap P_3, \Tout) \\
    &\leq  R(\largep, \OPT) + 2R(\largep \cap P_3, \Tout).
  \end{align*}
  Combined with the upper bound on $R(\largep \cap P_3,\Tout)$ from \lemref{T'approxca}, we get the statement of the lemma.
\end{proof}

This is the final lemma required for proving \thmref{gen-first-algorithm}. Thus, this completes the proof of the main theorem, \thmref{final-algorithm}.

\end{document}

\end{document}

%% file: prel.tex
\usepackage{amsmath,amsfonts,stmaryrd,amsthm,amssymb,enumitem}
\usepackage{algorithm,booktabs,url,hyperref}
\usepackage[color=yellow]{todonotes}

\usepackage{algorithmic}

\usepackage{pifont}

\usepackage{tikz}
\usetikzlibrary{arrows}
\usetikzlibrary{calc}
\usetikzlibrary{shapes}
\usetikzlibrary{fit}
\usetikzlibrary{matrix} 
\usetikzlibrary{positioning}
\usetikzlibrary{decorations.pathreplacing}

\newtheorem{theorem}{Theorem}[section]
\newtheorem{lemma}[theorem]{Lemma}

\newtheorem{definition}[theorem]{Definition}

\renewcommand{\eqref}[1]{Eq.~(\ref{eq:#1})}
\newcommand{\figref}[1]{Figure~\ref{fig:#1}}
\newcommand{\tabref}[1]{Table~\ref{tab:#1}}        
\newcommand{\secref}[1]{Section~\ref{sec:#1}}
\newcommand{\thmref}[1]{Theorem~\ref{thm:#1}}
\newcommand{\lemref}[1]{Lemma~\ref{lem:#1}}

\newcommand{\appref}[1]{Appendix~\ref{ap:#1}}

\newcommand{\algref}[1]{Alg.~\ref{alg:#1}}

\input{prel_def.tex}

\input{prel_cal.tex}

%% file: prel_def.tex
\renewcommand{\P}{\mathbb{P}}
\newcommand{\E}{\mathbb{E}}

\newcommand{\reals}{\mathbb{R}}
\newcommand{\nats}{\mathbb{N}}

\newcommand{\floor}[1]{\left\lfloor #1\right\rfloor}

\DeclareMathOperator*{\argmin}{argmin}



%% file: prel_cal.tex
\newcommand{\cA}{\mathcal{A}}
\newcommand{\cB}{\mathcal{B}}

%% file: Arxiv_paper.bbl
\begin{thebibliography}{20}
\providecommand{\natexlab}[1]{#1}
\providecommand{\url}[1]{\texttt{#1}}
\expandafter\ifx\csname urlstyle\endcsname\relax
  \providecommand{\doi}[1]{doi: #1}\else
  \providecommand{\doi}{doi: \begingroup \urlstyle{rm}\Url}\fi

\bibitem[Ackermann et~al.(2012)Ackermann, M{\"a}rtens, Raupach, Swierkot,
  Lammersen, and Sohler]{ackermann2012streamkm++}
Marcel~R Ackermann, Marcus M{\"a}rtens, Christoph Raupach, Kamil Swierkot,
  Christiane Lammersen, and Christian Sohler.
\newblock Stream{KM}++: A clustering algorithm for data streams.
\newblock \emph{Journal of Experimental Algorithmics (JEA)}, 17:\penalty0 2--4,
  2012.

\bibitem[Aggarwal(2003)]{aggarwal2003framework}
Charu~C Aggarwal.
\newblock A framework for diagnosing changes in evolving data streams.
\newblock In \emph{Proceedings of the 2003 ACM SIGMOD international conference
  on Management of data}, pages 575--586. ACM, 2003.

\bibitem[Ailon et~al.(2009)Ailon, Jaiswal, and Monteleoni]{ailon2009streaming}
Nir Ailon, Ragesh Jaiswal, and Claire Monteleoni.
\newblock Streaming k-means approximation.
\newblock In \emph{Advances in neural information processing systems}, pages
  10--18, 2009.

\bibitem[Bhaskara and Rwanpathirana(2020)]{bhaskara2020robust}
Aditya Bhaskara and Aravinda~Kanchana Rwanpathirana.
\newblock Robust algorithms for online $ k $-means clustering.
\newblock In \emph{Algorithmic Learning Theory}, pages 148--173, 2020.

\bibitem[Bhattacharjee and Moshkovitz(2021)]{bhattacharjee2020no}
Robi Bhattacharjee and Michal Moshkovitz.
\newblock No-substitution k-means clustering with adversarial order.
\newblock In \emph{Algorithmic Learning Theory}, pages 345--366. PMLR, 2021.

\bibitem[Braverman et~al.(2016)Braverman, Feldman, and Lang]{braverman2016new}
Vladimir Braverman, Dan Feldman, and Harry Lang.
\newblock New frameworks for offline and streaming coreset constructions.
\newblock \emph{arXiv preprint arXiv:1612.00889}, 2016.

\bibitem[Charikar et~al.(2002)Charikar, Guha, Tardos, and
  Shmoys]{charikar2002constant}
Moses Charikar, Sudipto Guha, {\'E}va Tardos, and David~B Shmoys.
\newblock A constant-factor approximation algorithm for the k-median problem.
\newblock \emph{Journal of Computer and System Sciences}, 65\penalty0
  (1):\penalty0 129--149, 2002.

\bibitem[Chen(2009)]{chen2009coresets}
Ke~Chen.
\newblock On coresets for k-median and k-means clustering in metric and
  euclidean spaces and their applications.
\newblock \emph{SIAM Journal on Computing}, 39\penalty0 (3):\penalty0 923--947,
  2009.

\bibitem[Dubhashi and Ranjan(1996)]{dubhashi1996balls}
Devdatt~P Dubhashi and Desh Ranjan.
\newblock Balls and bins: A study in negative dependence.
\newblock \emph{BRICS Report Series}, 3\penalty0 (25), 1996.

\bibitem[Guha et~al.(2000)Guha, Mishra, Motwani, and
  O'Callaghan]{guha2000clustering}
Sudipto Guha, Nina Mishra, Rajeev Motwani, and Liadan O'Callaghan.
\newblock Clustering data streams.
\newblock In \emph{Foundations of computer science, 2000. proceedings. 41st
  annual symposium on}, pages 359--366. IEEE, 2000.

\bibitem[Hess and Sabato(2020)]{hess2020sequential}
Tom Hess and Sivan Sabato.
\newblock Sequential no-substitution k-median-clustering.
\newblock In \emph{International Conference on Artificial Intelligence and
  Statistics}, pages 962--972. PMLR, 2020.

\bibitem[Joag-Dev and Proschan(1983)]{joag1983negative}
Kumar Joag-Dev and Frank Proschan.
\newblock Negative association of random variables with applications.
\newblock \emph{The Annals of Statistics}, pages 286--295, 1983.

\bibitem[Leung and Leckie(2005)]{leung2005unsupervised}
Kingsly Leung and Christopher Leckie.
\newblock Unsupervised anomaly detection in network intrusion detection using
  clusters.
\newblock In \emph{Proceedings of the Twenty-eighth Australasian conference on
  Computer Science-Volume 38}, pages 333--342. Australian Computer Society,
  Inc., 2005.

\bibitem[Li and Svensson(2016)]{li2016approximating}
Shi Li and Ola Svensson.
\newblock Approximating k-median via pseudo-approximation.
\newblock \emph{SIAM Journal on Computing}, 45\penalty0 (2):\penalty0 530--547,
  2016.

\bibitem[Liberty et~al.(2016)Liberty, Sriharsha, and
  Sviridenko]{liberty2016algorithm}
Edo Liberty, Ram Sriharsha, and Maxim Sviridenko.
\newblock An algorithm for online k-means clustering.
\newblock In \emph{2016 Proceedings of the Eighteenth Workshop on Algorithm
  Engineering and Experiments (ALENEX)}, pages 81--89. SIAM, 2016.

\bibitem[Meyerson et~al.(2004)Meyerson, O'callaghan, and
  Plotkin]{meyerson2004k}
Adam Meyerson, Liadan O'callaghan, and Serge Plotkin.
\newblock A k-median algorithm with running time independent of data size.
\newblock \emph{Machine Learning}, 56\penalty0 (1-3):\penalty0 61--87, 2004.

\bibitem[Moshkovitz(2021)]{moshkovitz2021unexpected}
Michal Moshkovitz.
\newblock Unexpected effects of online no-substitution $ k $-means clustering.
\newblock In \emph{Algorithmic Learning Theory}, pages 892--930. PMLR, 2021.

\bibitem[Nasraoui et~al.(2007)Nasraoui, Cerwinske, Rojas, and
  Gonzalez]{nasraoui2007performance}
Olfa Nasraoui, Jeff Cerwinske, Carlos Rojas, and Fabio Gonzalez.
\newblock Performance of recommendation systems in dynamic streaming
  environments.
\newblock In \emph{Proceedings of the 2007 SIAM International Conference on
  Data Mining}, pages 569--574. SIAM, 2007.

\bibitem[Shepitsen et~al.(2008)Shepitsen, Gemmell, Mobasher, and
  Burke]{shepitsen2008personalized}
Andriy Shepitsen, Jonathan Gemmell, Bamshad Mobasher, and Robin Burke.
\newblock Personalized recommendation in social tagging systems using
  hierarchical clustering.
\newblock In \emph{Proceedings of the 2008 ACM conference on Recommender
  systems}, pages 259--266. ACM, 2008.

\bibitem[Zheng et~al.(2014)Zheng, Yoon, and Lam]{zheng2014breast}
Bichen Zheng, Sang~Won Yoon, and Sarah~S Lam.
\newblock Breast cancer diagnosis based on feature extraction using a hybrid of
  k-means and support vector machine algorithms.
\newblock \emph{Expert Systems with Applications}, 41\penalty0 (4):\penalty0
  1476--1482, 2014.

\end{thebibliography}
